\RequirePackage{fix-cm}
\documentclass[smallextended]{svjour3}       
\smartqed  
\usepackage{amsmath}
\usepackage{amsfonts}
\usepackage{cite}
\usepackage{comment}
\usepackage{color}
\usepackage{amssymb}
\usepackage{mathrsfs}
\usepackage{graphicx}
\usepackage[font=scriptsize]{caption}
\usepackage[margin=3cm]{geometry}

\usepackage{kantlipsum} 

\usepackage{subfigure}

\usepackage[justification=centering]{caption}
 \usepackage{mathptmx}      
\newtheorem{assumption}[definition]{Assumption}

\begin{document}

\title{Lepskii Principle for Distributed Kernel Ridge Regression}


\author{Shao-Bo Lin 
}


\institute{S. B. Lin (sblin1983@gmail.com) \at Center for Intelligent Decision-Making and Machine Learning, School of Management, Xi'an Jiaotong University, Xi'an, China}

\date{Received: date / Accepted: date}

\maketitle

\begin{abstract}
Parameter selection without communicating local data is   quite challenging  in distributed learning, exhibing an inconsistency between theoretical analysis and practical  application of it in tackling distributively stored data. Motivated by the recently developed Lepskii principle and non-privacy communication protocol  for kernel learning, we propose  a Lepskii principle to equip distributed kernel ridge regression (DKRR) and consequently develop an adaptive DKRR with Lepskii principle (Lep-AdaDKRR for  short) by using a double weighted averaging synthesization scheme. We deduce optimal learning rates for Lep-AdaDKRR and theoretically  show that Lep-AdaDKRR succeeds in adapting to the regularity of regression functions, effective dimension decaying rate of  kernels and different metrics of   generalization, which fills the gap of the mentioned inconsistency between theory and application.

\end{abstract}

%


\section{Introduction}
With the growing awareness of privacy, the issue of data silos has become increasingly severe. Data silos \cite{patel2019bridging} are isolated collections of data that are accessible only to certain local agents and cannot be communicated with each other. In this way, data silos  
prevent   efficient data sharing and collaboration, making it a recent focus in machine learning and scientific computation to address data silos.

Distributed learning based on  divide-and-conquer schemes \cite{balcan2012distributed} is an innovative approach   enables multiple agents   to collaboratively train models without sharing  private information and thus is a promising approach to tackle data silos. Though theoretical results \cite{zhang2015divide,mucke2018parallelizing,shi2019distributed,linj12020optimal}  demonstrated its feasibility and efficiency, showing that distributed learning offers a promising solution to the growing concerns about data privacy,
adaptive parameter selection without sharing data in distributed learning remains open, exhibiting an inconsistency between theoretical analysis and real application and prohibiting the direct usage of distributed learning in settling distributive stored data. 
Our purpose  in this paper is to derive an adaptive parameter selection strategy without communicating private information  to equip distributed  learning to address the data silos.

The parameter selection issue is not pretty new in the realm of distributed learning. It was discussed in \cite{zhang2015divide,guo2017learning} that  optimal parameter of the local algorithm adopted in each local agent in distributed learning should be the same as that of its non-distributed version that tackles the whole data without the privacy consideration. In particular, a generalized cross-validation scheme has been proposed in \cite{xu2019distributed} to equip distributed learning, provided the data themselves can be communicated. When the data cannot shared across  local agents, a logarithmic mechanism based on cross-validation has been proposed in \cite{liu2022enabling}. Despite its  perfect theoretical assessments on the generalization performance, an assumption that the parameter behaves algebraically with respect to $|D|$ for all types of data is imposed, which is difficult to verify in practice.
For distributed kernel ridge regression (DKRR), a well developed distributed kernel-based learning scheme \cite{zhang2015divide,lin2017distributed,chang2017distributed}, 
an adaptive parameter selection based on non-private information communication protocols and cross-validation has been developed in \cite{wang2023adaptive} and optimal learning rates of the corresponding algorithms have been presented.  However, there are still three main drawbacks of the  approach  in \cite{wang2023adaptive}, which inevitably prevent the understanding the role of parameters  in the learning process. At first, the cross-validation approach in \cite{wang2023adaptive}
generally requires that  the empirical excess risk is an accessible unbiased estimate of the population risk,   prohibiting the application of it in deriving parameters for DKRR under the reproducing kernel Hilbert space (RHKS) norm. Then, there is not an explicit relation between the global estimate derived in \cite{wang2023adaptive}  and local estimates produced by local agents, making it difficult to quantify the role of   local estimates. Finally, a   truncation (or clipped) operator available to only bounded outputs \cite{steinwart2008support} is introduced, implying that the corresponding approach is only accessible to data with bounded noise.

This paper proposes a novel parameter selection strategy for DKRR according to the non-privacy communication protocol developed in \cite{wang2023adaptive} and
the well known Lepskii principle  adopted in  \cite{lepskii1991problem,lin2024adaptive}. There are mainly three ingredients in the selection process: local approximation based on specific basis functions, global approximation via communicating non-private information and early stopping determined by differences of successive global approximation and empirical effective dimensions. Different from the classical DKRR \cite{zhang2015divide,lin2017distributed,lin2020distributed} that are built upon a single weighted averaging synthesization, DKRR associated with the proposed Lepskii principle requires double weighted averaging to reflect the differences of optimal parameters on different local agents. With these, we propose a novel distributed learning scheme named as adaptive distributed kernel ridge regression with Lepskii principle (Lep-AdaDKRR for short) succeeds in addressing the  data silos. We present optimal generalization error estimates for Lep-AdaDKRR in the framework of learning theory \cite{cucker2007learning,steinwart2008support} and theoretically show that the proposed parameter selection strategy, without communicating any private information, adapts to the regularity of regression functions \cite{cucker2007learning,steinwart2008support}, decaying property of kernels measured by effective dimension \cite{lu2020balancing}
and approximation metrics \cite{blanchard2019lepskii} and is thus  promising in distributed learning.




\section{Adaptive DKRR with Lepskii Principle}
This section devotes to introducing DKRR, Lepskii principle and Lep-AdaDKRR.  

\subsection{DKRR and parameter selection}
Let $({\mathcal H}_K, \|\cdot\|_K)$ be the   RKHS induced by a
Mercer kernel $K$ on a  compact  metric space ${\mathcal X}$ with $\kappa:=\sup_{x\in\mathcal X}\sqrt{K(x,x)}<\infty$ and $D_j:=\{(x_{i,j},y_{i,j})\}_{i=1}^{|D_j|}\subset\mathcal X\times\mathcal Y$ with  
$\mathcal Y\subseteq\mathbb R$     be the set of data stored on the $j$-th agent for $j=1,2,\dots,m$, where $|A|$ denotes the cardinality of the set $A$.
Without loss of generality, we assume $D_j\cap D_{j'}=\varnothing$, meaning that there are no common samples for any two agents. Write   $D=\cup_{j=1}^mD_j$ as the total samples for the sake of convenience, though it cannot be accessed for any agents due to  data privacy.

For a given $\lambda>0$, DKRR is defined by  \cite{zhang2015divide,lin2017distributed}
\begin{equation}\label{DKRR-1}
    \overline{f}_{D,\lambda}=\sum_{j=1}^m\frac{|D_j|}{|D|}f_{D_j,\lambda},
\end{equation}
where 
\begin{equation}\label{KRR-local}
    f_{D_j,\lambda} :=\arg\min_{f\in \mathcal{H}_{K}}
    \left\{\frac{1}{|D_j|}\sum_{(x, y)\in D_j}(f(x)-y)^2+\lambda\|f\|^2_{K}\right\} 
\end{equation}
denotes the local estimate generated by the $j$-th local agent. Let $\mathbb K_j=\{K(x_{i,j},x_{i',j})\}_{i,i'=1}^{|D_j|}$ be the kernel matrix, it can be found in \cite{smale2007learning} \eqref{KRR-local} that
\begin{equation}\label{Analytic-KRR}
    f_{D_j,\lambda}= \sum_{i=1}^{|D_j|}a_{i,j}K_{x_i},\qquad \vec{a}_j=(a_{1,j},\dots,a_{|D_j|,j})^T=(\mathbb K_j+\lambda|D_j| I)^{-1}y_{D_j},
\end{equation}
where  $K_x:=K(x,\cdot)$ and $y_{D_j}:=(y_{1,j},\dots,y_{|D_j|,j})^T$. This depicts that there are totally $\mathcal O(|D_j|^3)$ floating computations to obtain a local estimate    and consequently $\mathcal O(\sum_{j=1}^m|D_j|^3)$ floating computations to obtain a global estimate of DKRR with given regularization parameter $\lambda$. Due to this, DKRR has originally been regarded as an acceleration version \cite{zhang2015divide,lin2017distributed,lin2020distributed} of the traditional kernel ridge regression (KRR)
\begin{equation}\label{KRR}
    f_{D,\lambda} :=\arg\min_{f\in \mathcal{H}_{K}}
    \left\{\frac{1}{|D|}\sum_{(x, y)\in D}(f(x)-y)^2+\lambda\|f\|^2_{K}\right\}, 
\end{equation}
which needs   $\mathcal O(|D|^3)$ computation complexity to obtain an estimate. 

If the parameter parameter is appropriately selected, the optimal learning rates of $\overline{f}_{D,\lambda}$ has been verified in \cite{zhang2015divide,lin2017distributed,lin2020distributed,wang2023adaptive} to be similar as that of $f_{D,\lambda}$ in the framework of learning theory, in which  the  samples in $D_j$, $j=1,\dots,m$, are assumed to be  drawn identically and independently (i.i.d.) according to an unknown but definite  distribution $\rho=\rho(y|x)\times \rho_X$ with $\rho(y|x)$ the   distribution conditioned on $x$ and $\rho_X$ the marginal distribution with respect  to $x$. The target of learning is to find an estimate based on $D_j,j=1,\dots,m$ to well approximate the regression function $f_\rho:=E[y|x]$. If the whole data set $D$ is accessible, i.e., there is not any privacy and communication considerations, it is not difficult to develop strategies such as cross-validation and generalized cross-validation  to equip DKRR, just as \cite{xu2019distributed} did. However, if data in each local agent cannot be communicated with each other, parameter selection becomes quite challenging, mainly due to the following error decomposition derived in \cite{chang2017distributed,guo2017learning}.
 
\begin{lemma}\label{Lemma:classical-error-dec}
Let $\overline{f}_{D,\lambda}$ be defined by \eqref{DKRR-1}. Then
\begin{equation}\label{classical-error-dec}
    E[\|\overline{f}_{D,\lambda}-f_\rho\|_*^2]
    \leq
    \sum_{j=1}^m\frac{|D_j|^2}{|D|^2}E[\|f_{D_j,\lambda}-f_\rho\|_*^2]+
    \sum_{j=1}^m\frac{|D_j|}{|D|}E[\|f^\diamond_{D_j,\lambda}-f_\rho\|_*^2],
\end{equation}
where $\|\cdot\|_*$ denotes either $\|\cdot\|_\rho$ or $\|\cdot\|_K$ and $f^\diamond_{D_j,\lambda}$ is the noise free version of $f_{D_j,\lambda}$ given by
\begin{equation}\label{KRR:noise-free}
      f^\diamond_{D_j,\lambda} :=\arg\min_{f\in \mathcal{H}_{K}}
    \left\{\frac{1}{|D_j|}\sum_{(x, y)\in D_j}(f(x)-f_\rho(x))^2+\lambda\|f\|^2_{K}\right\}. 
\end{equation}
\end{lemma}
 
As shown in \eqref{classical-error-dec}, the generalization error of $\overline{f}_{D,\lambda}$ can
be divided into two terms, each of which performs better than the local estimate $f_{D_j,\lambda}$. In particular, there is an addition $\frac{|D_j|}{|D|}$ in the first term of the r.s.h of \eqref{classical-error-dec} to show the power of synthesization, while the second term, though utilize only $|D_j|$ samples, focuses only on noise-free data. To derive comparable generalization errors of $\overline{f}_{D,\lambda}$ and $f_{D,\lambda}$, it is necessary to made $\sum_{j=1}^m\frac{|D_j|}{|D|}E[\|f^\diamond_{D_j,\lambda}-f_\rho\|_*^2]$  comparable with  $E[\|f_{D,\lambda}-f_\rho\|_*^2]$, requiring  strict restrictions on $|D_j|, j=1,\dots,m$. The problem is, however, that the selection of $\lambda$ minimizing $  \|f^\diamond_{D_j,\lambda}-f_\rho\|_*$ is much smaller that  minimizing $\|f_{D_j,\lambda}-f_\rho\|_*$, and cannot be numerically found directly.
  In a word, if data in $D_j$ cannot communicated with each other, then  it is difficult to adopt   the existing parameter selection strategies such as cross-validation \cite{caponnetto2010cross} and generalized cross-validation \cite{xu2019distributed} to equip DKRR.

\subsection{Lepskii principle for DKRR}
Our purpose is to utilize the recently developed Lepskii principle (also called as the balancing principle) \cite{lu2020balancing,blanchard2019lepskii,lin2024adaptive} to equip DKRR to fill the gap between theory and applications. Lepskii principle that  selects parameter by bounding  differences of two successive estimates  was firstly adopted in \cite{de2010adaptive} for the learning purpose to determine the regularization parameter of KRR and then improved in \cite{lu2020balancing} to involve the capacity information of RKHS and \cite{blanchard2019lepskii} to adapt to different metrics of generalization. In our recent work \cite{lin2024adaptive}, we proposed a novel Lepskii principle to select regularization parameter for KRR to remove the recurrent comparisons among different parameters in the existing literature \cite{lu2020balancing,blanchard2019lepskii} and derived optimal learning rates for corresponding KRR that is a new record for selection parameter in kernel methods. 
To introduce the Lepskii principle, we need the following mild assumption abounding in the literature \cite{caponnetto2007optimal,blanchard2018optimal,blanchard2019lepskii,lin2018distributed,lin2024adaptive}. 
\begin{assumption}\label{Assumption:boundedness}
   Assume $\int_{\mathcal Y}
y^2d\rho<\infty$ and
\begin{equation}\label{Boundedness for output}
\int_{\mathcal Y}\left(e^{\frac{|y-f_\rho(x)|}M}-\frac{|y-f_\rho(x)|}M-1\right)d\rho(y|x)\leq
\frac{\gamma^2}{2M^2}, \qquad \forall x\in\mathcal X,
\end{equation}
where $M$ and $\gamma$ are positive constants.
\end{assumption}
  It can be found in   
\cite{bauer2007regularization} that (\ref{Boundedness for output}) is equivalent
to the following momentum condition (up to a change of constants)
$$
           \int_{\mathcal
           Y}|y-f_\rho(x)|^{\ell}d\rho(y|x)\leq\frac12\ell
           !\gamma^2M^{\ell-2}, \qquad \forall \ell\geq 2, x\in\mathcal X 
$$
and therefore  is satisfied when the noise is uniformly
bounded, Gaussian or sub-Gaussian \cite{raskutti2014early}. In  particular, if $|y_{i,j}|\leq M'$ almost surely for some $M'>0$,  then \eqref{Boundedness for output} holds naturally  with 
$\gamma/2 = M = M'$.

In each local agent, define  
 \begin{eqnarray}\label{Def.WD}
\mathcal W_{D_j,\lambda}:=
 \frac{1}{|D_j|\sqrt{ \lambda}}+\left( 1+ \frac{1}{\sqrt{\lambda|D_j|}}\right) \sqrt{\frac{\max\{\mathcal{N}_{D_j}(\lambda),1\}}{|D_j|}},
\end{eqnarray}
 where 
\begin{equation}\label{Definition-empi-effec}
   \mathcal N_{D_j}(\lambda):
   ={\rm Tr}[(\lambda|D_j|I+\mathbb K_j)^{-1}\mathbb K_j],\qquad \forall\ \lambda>0 
\end{equation}
denotes the   empirical effective dimension and and  ${\rm Tr}(A)$ is the trace of the matrix (or operator) $A$. 
   Denote
\begin{equation}\label{Def.function-differenction}
    g_{D_j,\lambda_{k}}:=f_{D_j,\lambda_{k}}-f_{D_j,\lambda_{k-1}} 
\end{equation}
as the difference of successive estimate with \begin{equation}\label{def:lambda-selection}
    \lambda_{k}:=\frac1{kb},\qquad b\geq 1,  \quad k=1,2, \dots.
\end{equation}
The Lepskii principle proposed in \cite{lin2024adaptive} starts with an upper bound of $k$ on the $j$-th  local agent,
$$
    K_j:=\left[\frac{|D_j|}{16b(C_1^*)^2\log^3(16|D_j|)}\right],\qquad j=1,\dots,m
$$
where $[a]$ denotes the integer part of the positive number $a$ and $C_1^*:=\max\{(\kappa^2+1)/3,2\sqrt{\kappa^2+1}\}$. Then selection $k:=K_j,K_j-1,\dots,1$ be first (or largest) integer such that 
\begin{equation}\label{stopping-LEP}
  \|(L_{K,D_j}+\lambda_{k-1} I)^{1/2}(f_{\lambda_{k},D_j}-f_{\lambda_{k-1},D_j})\|_{K}\geq
C_{LP}^* \lambda_{k-1}\mathcal W_{D_j, \lambda_{k}}
\end{equation}
for some computable constant $C_{LP}^*$ depending only on $M,\gamma,b$\footnote{Since we are concerned with estimate in expectation, the logarithmic term associated with the confidence level is removed. As a result, the constant $C_{LP}^*$ is different from $C_{US}$ in \cite{lin2024adaptive}.}, where $L_{K,D_j}:\mathcal H_K\rightarrow\mathcal H_K$ is defined by
$$
  L_{K,D_j}f:=\sum_{(x,y)\in D_j}f(x)K_x.
$$
Due to the above definition, it is easy to derive \cite{blanchard2019lepskii,lin2024adaptive}
$$
   \|(L_{K,D_j}+\lambda_k I)^{1/2}(f_{\lambda_{k},D_j}-f_{\lambda_{k-1},D_j})\|_{K}
   =\left(\|f_{D_j,\lambda_k}-f_{D_j,\lambda_{k-1}}\|_{D_j}^2+\lambda_{k}\|f_{D_j,\lambda_k}-f_{D_j,\lambda_{k-1}}\|_K^2\right)^{1/2},
$$
where $\|f\|_{D_j}^2=\frac1{|D|}\sum_{i=1}^{|D_j|}|f(x_{i,j})|^2$ and for $f_{D_j,\lambda_k}=\sum_{i=1}^{|D_j|}\alpha_{i,k}K_{x_{i,j}}$ 
$$
   \|f_{_j,\lambda_k}-f_{D_j,\lambda_{k-1}}\|_K^2
   =\sum_{i,i'=1}^{|D_j|}(\alpha_{i,k}-\alpha_{i,k-1})(\alpha_{i,k}-\alpha_{i',k-1})K(x_{i,j},x_{i',j}).
$$ 
All these demonstrate that both sides in \eqref{stopping-LEP} are computable, making the parameter selection be numerically feasible. 

The learning performance of KRR with the mentioned Lepskii principle \eqref{stopping-LEP} has been verified in \cite{lin2024adaptive} to be optimal in the framework of learning theory. However, it cannot be adopted directly to select the regularization parameter for DKRR. In fact, as shown in Lemma \ref{Lemma:classical-error-dec}, besides an upper bound for $\|f_{D_j,\lambda}-f_\rho\|_*$ whose parameter can be determined by \eqref{stopping-LEP}, it also requires to bound $\|f^\diamond_{D_j,\lambda}-f_\rho\|_*$, in which $f^\diamond_{D_j,\lambda}$ is unknown in practice and thus we cannot use  \eqref{stopping-LEP} to determine a suitable $\lambda$.  

In this paper, we borrow the ``non-parameter$\rightarrow$parameter$\rightarrow$non-parameter'' protocol from \cite{wang2023adaptive} to develop a  variant of Lepskii principle to feed DKRR.  The ``non-parameter$\rightarrow$parameter'' step is to generate the same set of   basis functions on both local agent and global agent and use the linear combination of these basis functions to approximate the local estimate $f_{D_j,\lambda}$ for given $\lambda$. Then local agents communicate the coefficients of linear combination to the global agent to synthesize a global  approximation of the global estimate $\overline{f}_{D,\lambda}$. The global agent then transmits back the coefficients of the global approximation back to each local agent. The  ``parameter$\rightarrow$non-parameter'' step is to select $\lambda$ by using the Lepskii principle like \eqref{stopping-LEP} based on the obtain parametric global approximation on each local agent.

\subsection{Adaptive DKRR with Lepskii principle} 

In this subsection, we propose a novel  adaptive DKRR algorithm without communicating data. We call it as the Adaptive DKRR with Lepskii principle (Lep-AdaDKRR for short). There are mainly five steps in Lep-AdaDKRR. 
 
$\bullet$ {\it Step 1. Center  generating and communications:} The $j$-th local agent  transmits the number of samples $|D_j|$ to the global agent and the global agent i.i.d. draws 
$
     L \geq \max_{j=1,\dots,m} |D_j|
$
points $\Xi_L:=\{\xi_\ell\}_{\ell=1}^L\subseteq \mathcal X$ according to the uniform distribution and transmits them as well as the value 
\begin{equation}\label{DefK*}
          K^*=\min_{j=1,\dots,m}\left[\frac{|D_j|}{16b(C_1^*)^2\log^3(16|D_j|)}\right].
\end{equation}
back to all local agent.  

$\bullet$ {\it Step 2. Local  approximation generating and communications:} In each local agent, a set of local estimates $\{f_{D_j,\lambda_{k}}\}_{k=1}^{K^*}$ defined by 
\eqref{KRR-local}
 with $\lambda_k$ satisfying \eqref{def:lambda-selection} and a set of quantities $\{\mathcal W_{D_j,\lambda_k}\}_{k=1}^{K^*}$ defined by \eqref{Def.WD} are generated.  Moreover, after receiving $L$ centers from the global agent, each local agent generates a set of basis functions by 
\begin{equation}\label{Basis-generation}
     B_{L,K}:=\left\{\sum_{\ell=1}^La_\ell K_{\xi_\ell}:a_\ell\in\mathbb R\right\}.
\end{equation}
Then, define the local approximation  of $g_{D_j,\lambda_{k}}$ satisfying \eqref{Def.function-differenction}  by
\begin{equation}\label{local-approximation}
g^{loc}_{D_j,\lambda_{k},\mu_j }:=\arg\min_{f\in B_{L,K}}\frac1{|D_j|}\sum_{i=1}^{|D_j|}(f(x_{i,j})
     -g_{D_j,\lambda_{k}}(x_{i,j}))^2+\mu_j \|f\|_K^2 =:\sum_{\ell=1}^La_{j,k,\ell,\mu_j }K_{\xi_\ell}
\end{equation}
with 
\begin{equation}\label{Selection-mu}
    \mu_j :=\frac{48bC_1^*\log(1+8\kappa |D_j|)}{|D_j|}  .
\end{equation}
The $j$-th local agent then
transmits the vector $\{\mathcal W_{D_j,\lambda_{k}}\}_{k=1}^{K^*}$ and  the
 coefficient matrix 
$\{a_{j,k,\ell,\mu_j }\}_{k=1,\ell=1}^{K^*,L}$
to the global agent.

$\bullet$ {\it Step 3. Global approximation generating and communications:}  The global agent  collects $j$ sets of coefficients and synthesizes them via weighted average, i.e.,
\begin{equation}\label{global-111}
    {a}^{global}_{k,\ell,\mu }
  :=\sum_{j=1}^m\frac{|D_j|}{|D|}a_{j,k,\ell,\mu_j },
  \qquad 
  \overline{\mathcal W}_{D,\lambda_k}:=\sum_{j=1}^{m}\frac{|D_j|^2}{|D|^2}
  \mathcal W_{D_j,\lambda_k}^2.
\end{equation}
Then the global agent transmits   the vector $\{\vec{\mathcal W}_{D,\lambda_k}\}_{k=1}^{K^*}$ and the coefficient matrix $\{ {a}^{global}_{k,\ell,\mu }\}_{k=1,\ell=1}^{K^*,L}$  to each local agent.

$\bullet$ {\it Step 4. Lepskii-type principle and communications:} Each local agent receives the   coefficients  $\{ {a}^{global}_{k,\ell,\mu }\}_{k=1,\ell=1}^{K^*,L}$  and generates the global approximation
 \begin{equation}\label{global-approximation-1}
    g^{global}_{D,\lambda_k,\mu}
   :=\sum_{j=1}^m\frac{|D_j|}{|D|}\sum_{\ell=1}^La_{j,k,\ell,\mu_j }K_{\xi_\ell}=\sum_{j=1}^m\frac{|D_j|}{|D|} 
   g^{loc}_{D_j,\lambda_{k},\mu_j }.
\end{equation}
For $k=K^*,K^*-1\dots,2$,
define ${k}_{j}^*$ to be the first $k$ satisfying 
\begin{equation}\label{stopping-1}
   \|(L_{K,D_j}+\lambda_k I)^{1/2}g^{global}_{D,\lambda_k,\mu}\|_K^2 
    \geq
    C_{LP}  \lambda_{k-1}^2\overline{\mathcal W}_{D,\lambda_k},
\end{equation}
where $C_{LP}:=4 b^2 (1+576\Gamma(5)(\kappa M +\gamma)^2(\sqrt{2}+4)^2 )$.
If there is not any $k$ satisfying the above inequality, define $k^*_{j}=K^*$.
Write $ {\lambda}^*_{j}=\lambda_{k^*_j}$. Given a query point $x^*$, each local agent transmits the matrix $\{f_{D_{j'},\lambda_j^*}(x^*)\}_{j,j'=1}^m$ to the global machine.  

$\bullet$ {\it Step 5. Global estimate generating:}
The global agent generates a global estimate via double weighted averaging, i.e.,
\begin{equation}\label{final-global}
    \overline{f}_{D, \vec{\lambda}^*}
    :=\sum_{j=1}^m\frac{|D_j|}{|D|}\overline{f}_{j}
    :=\sum_{j=1}^m\frac{|D_j|}{|D|}\sum_{j'=1}^m \frac{|D_{j'}|}{|D|}
      f_{D_{j'},\lambda^*_j},
\end{equation}
that is, based on the query point $x^*$, the global agent present a prediction $\overline{f}_{D, \vec{\lambda}^*}(x^*)$.
 
We then present several explanations  on the above procedures. In Step 1, a set of random variables needs to be communicated that may bring additional burden on communications. It should be highlighted that such a communication process can be avoided by using  low-discrepancy sequences such as Sobol sequences and Halton sequences, just as \cite{wang2023adaptive} did. For example, each local agent can generate the same Sobol sequences of size $L$. Our reason to use i.i.d. random variables and communicate them is only to ease the theoretical analysis. In Step 2,  we build the set of basis function $B_{L,K}$ as \eqref{Basis-generation} based on the kernel and center sets $\{\xi_\ell\}_{\ell=1}^L$. 
We believe that if the embedding space $\mathcal X$ is specified, $B_{L,K}$ can be replaced by some more simple  sets such as the set of polynomials, splines and wavelets. 
It should be highlighted that if the communication costs are not considered, larger $L$ and consequently smaller $\mu_j $ lead  to better performance of the local approximation.  In Step 4, a Lepskii principle similar as \eqref{stopping-LEP} is introduced for DKRR, which is pretty novel in the realm of distributed learning. Though we have present a detailed value of  $C_{LP}$ in \eqref{stopping-1}, its estimate is a little bit pessimistic. In practice, since $C_{LP}$ is independent of $|D_j|,j=1,\dots,m$, we can select part of the samples, i.e.,  the first $1/10$ samples from $D_j$, to adaptively determine   it by  cross-validation. In Step 5, we use a double weighted averaging scheme rather than the classical single weighted averaging scheme as \eqref{DKRR-1}. The first weighted averaging in \eqref{final-global} is made due to the global-approximation-based parameter selection strategy \eqref{stopping-1} to guarantee  the global approximation can  represent DKRR with fixed $\lambda$ well, while the second weighted averaging is to embody the difference of $D_j$ to show that different data sets correspond  to different regularization parameters.
In particular, if we focus on DKRR with fixed $\lambda_1=\dots=\lambda_m$, then the second weighted averaging can be removed.

\section{Theoretical Verifications}

In this section, we present theoretical verifications on the developed Lepskii principle for distributed learning in terms of its adaptivity  to the kernels, regression functions and error measurements. To show its adaptivity to regression functions, we need a well known regularity assumption, which is based on the 
 integral operator $L_K:L_{\rho_X}^2\rightarrow L_{\rho_X}^2$ (or $\mathcal H_K\rightarrow\mathcal H_K$ if no confusion is made) defined by
$$
           L_Kf=\int_{\mathcal X}f(x)K_xd\rho_X.
$$

\begin{assumption}\label{Assumption:regularity}
For $r>0$, assume
\begin{equation}\label{regularitycondition}
         f_\rho=L_K^r h_\rho,~~{\rm for~some}  ~ h_\rho\in L_{\rho_X}^2,
\end{equation}
where $L_K^r$  is defined by spectral calculus.
\end{assumption}

It is easy to see that Assumption \ref{Assumption:regularity}, abounding in the learning theory literature \cite{smale2007learning,caponnetto2007optimal,blanchard2016convergence,lin2017distributed,mucke2018parallelizing,guo2023optimality,lin2024adaptive}, 
describes the regularity of the regression function $f_\rho$ by showing that larger index $r$ implies better regularity of $f_\rho$. In particular, \eqref{regularitycondition} with $r=1/2$ is equivalent to $f_\rho\in\mathcal H_K$ while
\eqref{regularitycondition} with  $r>1/2$ implies $f_\rho$ is in a proper subset of $\mathcal H_K$. 
Different from the parameter selection strategy presented in \cite{raskutti2014early} that assumes $f_\rho\in\mathcal H_K$, i.e., \eqref{regularitycondition} with $r=1/2$, our purpose is to show that the proposed Lepskii principle for DKRR can adaptively reflect the regularity of $f_\rho$ in terms that it is available to \eqref{regularitycondition} with   $r\geq 1/2$.  

To adaptively embody the information of kernel, the 
effective dimension defined by
$$
        \mathcal{N}(\lambda)={\rm Tr}((\lambda I+L_K)^{-1}L_K),  \qquad \lambda>0 
$$
should be given. The following assumption quantifies the kernel via the effective dimension.

\begin{assumption}\label{Assumption:effective dimension}
 There exists some $s\in(0,1]$ such that
\begin{equation}\label{assumption on effect}
      \mathcal N(\lambda)\leq C_0\lambda^{-s},
\end{equation}
where $C_0\geq 1$ is  a constant independent of $\lambda$.
\end{assumption}

The above assumption widely adopting in the learning theory literature \cite{blanchard2016convergence,blanchard2019lepskii,lin2020distributed} was proved in \cite{fischer2020sobolev} to be equivalent to the well known eigenvalue decaying assumption \cite{caponnetto2007optimal,steinwart2009optimal,raskutti2014early}:
$$
    \sigma_k\leq C_0'k^{-\frac1s},
$$
where $(\sigma_k,\phi_k)$ is the eigen-pairs of $L_K$ with $\sigma_1\geq\sigma_2\geq\cdots$.
It is obvious that \eqref{assumption on effect} always holds for $s=0$ and $C_0'=\kappa$, and the index $s$ 
quantitatively reflects the smoothness of the kernel $K$. Based on the above assumptions, we are in a position to present our main results in the following theorem. 
 
\begin{theorem}\label{Theorem:Optimal-Rate-adaptive}
Under Assumption \ref{Assumption:boundedness}, Assumption \ref{Assumption:regularity} with $\frac12\leq r\leq 1$ and Assumption \ref{Assumption:effective dimension} with $0\leq s\leq 1$, if
  \begin{equation}\label{res.theorem}
      |D_j|\geq \bar{C}_0|D|^{\frac1{2r+s}}\log^3|D|  
  \end{equation}
for some $\bar{C}_0$ depending only $r,s$, and $\mu_j $ satisfies \eqref{Selection-mu} with $L\geq\max\{|D_1|,\dots,|D_m|\}$,
then
  \begin{eqnarray}\label{Final-rho-bound}
      E[\| \overline{f}_{D,\vec{\lambda}^*}-f_\rho\|_\rho^2]
     \leq \bar{C}|D|^{-\frac{2r}{2r+s}} 
\end{eqnarray} 
and
  \begin{eqnarray}\label{Final-K-bound}
      E[\|\overline{f}_{D,\vec{\lambda}^*}-f_\rho\|_K^2]
     \leq \bar{C}|D|^{-\frac{2r-1}{2r+s}}, 
\end{eqnarray} 
where $\bar{C}$   specified in the proof  is a constant depending only on $M,\|h\|_\rho,\gamma,s,r,\kappa$.
\end{theorem}

The proof of Theorem \ref{Theorem:Optimal-Rate-adaptive} can be found in the next section. It can be found in \cite{caponnetto2007optimal,fischer2020sobolev} that the derived learning rates in \eqref{Final-rho-bound} and \eqref{Final-K-bound} are rate-optimal in the sense that there are some $\rho^*$ satisfying Assumption \ref{Assumption:boundedness}, Assumption \ref{Assumption:regularity} and Assumption \ref{Assumption:effective dimension} such that 
$$
    E[\| \overline{f}_{D,\vec{\lambda}^*}-f_{\rho^*}\|_{\rho^*}^2]\geq \bar{C}_2|D|^{-\frac{2r}{2r+s}},\qquad  E[\| \overline{f}_{D,\vec{\lambda}^*}-f_{\rho^*}\|_K^2]\geq \bar{C}_2|D|^{-\frac{2r-1}{2r+s}}.
$$
Recalling   \cite{caponnetto2007optimal,fischer2020sobolev} again that the above lower bounds actually hold for any learning algorithms derived based on $D=\cup_{j=1}^mD_j$, we obtain that the proposed Lep-AdaDKRR is one of the most efficient learning algorithms to tackle such data. Since there is not any privacy information communicated in the learning process, Theorem \ref{Theorem:Optimal-Rate-adaptive} also demonstrates that Lep-AdaDKRR is one of the optimal learning algorithms to address data silos and consequently circumvents the inconsistency between theory and application of DKRR.  

Different from the classical theoretical analysis on distributed learning \cite{zhang2015divide,lin2017distributed,mucke2018parallelizing} that imposes restriction on $m$ and  assumes the equal size of data blocks, i.e.,$|D_1|\sim\dots\sim|D_m|$, \eqref{res.theorem} pays attention to the size of data block and removes the equal size assumption. Such a restriction presents the qualification for nomination of local agent to take part in  the distributed learning process.  Due to \eqref{res.theorem}, it is easy to derive  
\begin{equation}\label{res.m}
   m\leq \frac{|D|^{\frac{2r+s-1}{2r+s}}}{\bar{C}_0\log^3|D|},
\end{equation}
which is almost the same as that for DKRR with theoretically optimal regularization parameter \cite{lin2020distributed} under the equal size assumption. 

Compared with \cite{wang2023adaptive} where a cross-validation approach is proposed to equip DKRR, Lep-AdaDKRR possesses at least four advantages. At first, Lep-AdaDKRR does not require to divide the  sample into training sets and validation sets, which inevitably improves the utilization rate of samples. Then, besides the adaptivity to $r,s$, \eqref{Final-rho-bound} and \eqref{Final-K-bound} illustrate that Lep-AdaDKRR also adapts to the embodying metric, which is beyond the performance of cross-validation. It  should be   highlighted that  analysis under the $\mathcal H_K$ norm actually shows that Lep-AdaDKRR is capable of tackling covariant shifts problem \cite{ma2023optimally}, where the distributions of training and testing samples are different. Thirdly, there is not an explicit relation between the global estimate and local estimates in the approach in \cite{wang2023adaptive}, making it difficult to quantify the role of local agents. Differently,   Lep-AdaDKRR succeeds in derive a global estimate via double weight averaging of local estimates. Finally, there is a truncation operator in the estimate of \cite{wang2023adaptive} which not only imposes strict boundedness assumption of the noise but also deports the global estimate from $\mathcal H_K$. However, such a truncation operator is not needed in Lep-AdaDKRR. 
We finally present two remarks to end this section.

\begin{remark}\label{Remark:1}
There is a restriction $\frac12\leq r\leq 1$ in Theorem \ref{Theorem:Optimal-Rate-adaptive}, showing that Lepskii-AdaDKRR cannot fully exploit all regularity degrees  $0<r<\infty$. The reason is  due to DKRR itself rather than the proposed Lepskii principle. In fact, it can be found in \cite{gerfo2008spectral,guo2017learning} that KRR as well as DKRR  suffers from the so-called saturation phenomenon in the sense that it only  realizes the regularity degree satisfying $0<r\leq 1$.
For $0<r<1/2$, it even remains open to derive optimal learning rates for any distributed learning algorithms without strict restriction on $s$.    
\end{remark}

\begin{remark}
Several learning algorithms   such as the kernel gradient descent \cite{lin2018distributed}, kernel spectral cut-off \cite{dicker2017kernel}, and boosted KRR \cite{lin2019boosted} have been developed to circumvent the saturation phenomenon of KRR. The problem is,  as discussed in \cite{lin2024adaptive}, that the proposed Lepskii principle  is based on special spectral property of KRR  that no longer holds for other algorithms. If  the Lepskii-type principle in \cite{blanchard2019lepskii} that is available to any kernel-based spectral algorithms is used, we believe that it is feasible to design some adaptive version of distributed learning algorithms. However, the learning rates are no more optimal, just as \cite{lin2024adaptive} exhibited for non-distributed learning scheme. It would be interesting to design suitable parameter selection strategy to equip more general learning algorithms rather than DKRR to achieve optimal learning rates.
\end{remark}

\section{Proofs}
 In this section, we prove Theorem \ref{Theorem:Optimal-Rate-adaptive} by adopting the widely used integral operator approach \cite{smale2004shannon,smale2005shannon,lin2017distributed}.  the main novelty of our proof is to divide local agents into two categories: the one  is the type of local agents, called as the under-estimated agents, in which  the  $\lambda_{k^*_j}$ selected via \eqref{stopping-1} is smaller than $\lambda_0:= \lambda_{k_0}$ with $k_0=\left[|D|^\frac{1}{2r+s}\right]$
  and the other is the type of local agents, named as over-estimated agents,  in which   $\lambda_{k^*_j}\geq \lambda_0$. We use \eqref{stopping-1} for $k=k_j^*$ in deriving generalization error of Lep-AdaDKRR in under-estimated agents and
\begin{equation}\label{stopping-negative}
     \|(L_{K,D_j}+\lambda_k I)^{1/2}g^{global}_{D,\lambda_k,\mu}\|_K^2 
    \leq
    C_{LP}  \lambda_{k-1}^2\overline{\mathcal W}_{D,\lambda_k},\qquad k\geq k_j^*
\end{equation}
in deducing generalization error of Lep-AdaDKRR in over-estimated agents.

\subsection{Operator perturbation and operator representation}

We at first recall some basic concepts of operators from \cite{bhatia2013matrix}.
Let $\mathcal H$ be a Hilbert space and $A:\mathcal H\rightarrow\mathcal H$ be a bounded linear operator with its spectral norm defined by 
$$
        \|A\| =\sup_{\|f\|_{\mathcal H}\leq 1}\|Af\|_{\mathcal H}.
$$
If $A=A^T$,  then $A$ is  
self-adjoint, where $A^T$ denotes its adjoint. A self-adjoint operator is said to be positive if $\langle f,Af\rangle_{\mathcal H}\geq 0$ for all $f\in\mathcal H$. 
We then introduce three important properties of positive operators. The first one proved in \cite[Proposition 6]{rudi2015less} provides an upper bounds for operator perturbation. 

\begin{lemma}\label{Lemma:operator inequality general}
Let $\mathcal H,\mathcal K$ be two separable Hilbert spaces,
$A:\mathcal H\rightarrow\mathcal H$ be a positive linear operator,
$V_{H,K}:\mathcal H\rightarrow\mathcal K$ be a partial isometry and
$B:\mathcal K\rightarrow\mathcal K$ be a bounded operator. Then for
all $0\leq r^*,s^*\leq1/2$, there holds
$$
      \|A^{r^*}V_{HS}BV_{HS}^TA^{s^*}\|\leq\|(V_{H,K}^TAV_{H,K})^{r^*}B(V_{H,K}^TAV_{H,K})^{s^*}\|.
$$
\end{lemma}

 The second  one derived in \cite[Proposition 3]{rudi2015less} presents an important property of the projection operator.

\begin{lemma}\label{Lemma:Projection general}
Let $\mathcal H$, $\mathcal K$ , $\mathcal F$ be three separable
Hilbert spaces. Let $Z:\mathcal H\rightarrow\mathcal K$ be a bounded
linear operator and $P$ be a projection operator on $\mathcal H$
such that $\mbox{range}P=\overline{\mbox{range}Z^T}$. Then for any
bounded linear operator $F:\mathcal F\rightarrow\mathcal H$ and any
$\lambda>0$ we have
$$
     \|(I-P)F\|\leq\lambda^{1/2}\|(Z^TZ+\lambda I)^{-1/2}F\|.
$$
\end{lemma}

The third one is the well known Cordes inequality that can be found in Theorem IX.2.1 in \cite{bhatia2013matrix}.
\begin{lemma}\label{Lemma:cordes-11}
     Let $A, B$ be positive definite matrices. Then  
$$ 
     \left\|A^u B^u\right\| \leq\|A B\|^u,\qquad 0\leq u\leq 1.
$$
\end{lemma}

We then present the operator representation for DKRR, local approximation and global approximation.
 For any $j=1,\dots,m$, write the sampling operator
  $S_{D_j}:\mathcal H_K\rightarrow\mathbb R^{|D_j|}$ as
$$
         S_{D_j}f:=\{f(x_{i,j})\}_{(x_{i,j},y_{i,j})\in D_j}.
$$
Its scaled adjoint $S_{D_j}^T:\mathbb R^{|D_j|}\rightarrow \mathcal H_K$ is
$$
       S_{D_j}^T{\bf c}:=\frac1{|D_j|}\sum_{(x_{i,j},y_{i,j})\in D}c_iK_{x_{i,j}},\qquad {\bf c}=(c_1,\dots,c_{|D_j|})^T\in\mathbb R^{|D_j|}.
$$
Therefore, $\frac1{|D_j|}\mathbb K_j=S_{D_j}S_{D_j}^T$, 
$
  L_{K,D_j}=S_{D_j}^TS_{D_j}.
$
 and $L_{K,D_j}:\mathcal H_K\rightarrow\mathcal H_K$ is a positive operator.
Due to these, we have \cite{smale2007learning}
 \begin{equation}\label{KRR-local:operator}
         f_{D_j,\lambda}=(L_{K,D_j}+\lambda I)S^T_{D_j}y_{D_j},\qquad \mbox{and}\quad   f^\diamond_{D_j,\lambda}=(L_{K,D_j}+\lambda I)^{-1}L_{K,D_j}f_\rho.
\end{equation}
The following lemma provides several important bounds concerning $ f_{D_j,\lambda}$ and $ f^\diamond_{D_j,\lambda}$.

\begin{lemma}\label{Lemma:difference-krr}
If Assumption \ref{Assumption:regularity} holds with $\frac12\leq r\leq 1$, then   
\begin{eqnarray}
    \|(L_K+\lambda I)^{1/2}(f_{D_j,\lambda}^\diamond-f_\rho)\|_K
   &\leq&
   \mathcal Q_{D_j,\lambda}^{2r}\lambda^r\|h_\rho\|_\rho,\label{bound.approximation-error}\\
    \|(L_K+\lambda I)^{1/2}(f_{D_j,\lambda}^\diamond-f_{D_j,\lambda})\|_K
    &\leq&
    \mathcal Q_{D_j,\lambda}^2\mathcal P_{D_j,\lambda},\label{bound.sample-error}\\
     \|(L_{K,D_j}+\mu_j  I)^{-1/2}(f_{D_j,\lambda}-f_{D_j,\lambda'})\|_K
   & \leq &
   \mu_j ^{-1/2}\frac{|\lambda-\lambda'|}{\lambda'\sqrt{\lambda}}\mathcal Q_{D_j,\lambda}\mathcal P_{D_j,\lambda}
  +
   \mathcal Q_{D_j,\mu_j }^{2r-1}\frac{|\lambda-\lambda'|}{\lambda} \mu_j ^{r-1}\|h_\rho\|_\rho, \label{bound.error-difference}
\end{eqnarray} 
 where $\lambda,\lambda'>0$, and for some set of points $\Xi$,
\begin{eqnarray}
    \mathcal Q_{\Xi,\lambda} &:=& \|(L_K+\lambda I)^{1/2}(L_{K,\Xi}+\lambda I)^{-1/2}\|,\label{Def.QD}\\
    \mathcal P_{\Xi,\lambda}&:=&
	\left\|(L_K+\lambda
	I)^{-1/2}(L_{K,\Xi}f_\rho-S^T_{\Xi}y_{\Xi})\right\|_K. \label{Def.PD} 
\end{eqnarray}
\end{lemma}

\begin{proof}
The bounds in \eqref{bound.approximation-error} and \eqref{bound.sample-error} are standard and can be found in \cite{caponnetto2007optimal,lin2017distributed,lin2020distributed}.   Since 
\begin{equation}\label{Different-abc}
    A^{-1}-B^{-1}=A^{-1}(B-A)B^{-1}
\end{equation}
for positive operator $A,B$, we get
\begin{eqnarray*} 
    &&f_{D_j,\lambda}-f_{D_j,\lambda'} = ((L_{K,D_j}+\lambda I)^{-1}-(L_{K,D_j}+\lambda'I)^{-1})S_{D_j}^Ty_{D_j} \nonumber\\
  &=&
  (L_{K,D_j}+\lambda I)^{-1}(\lambda'-\lambda)[(L_{K,D_j}+\lambda'
  I)^{-1}(S_{D_j}^Ty_{D_j}-L_{K,D_j}f_\rho) +(L_{K,D_j}+\lambda' I)^{-1}L_{K,D_j}f_\rho].
\end{eqnarray*}
Then,   
Assumption \ref{Assumption:regularity} with $\frac12\leq r\leq 1$, \eqref{Def.PD}, \eqref{Def.QD} and   Lemma \ref{Lemma:cordes-11}  yield
\begin{eqnarray*}
     &&\|(L_{K,D_j}+\mu_j  I)^{-1/2}(f_{D_j,\lambda}-f_{D_j,\lambda'})\|_K\\
      &\leq &
    |\lambda-\lambda'|\|(L_{K,D_j}+\mu_j  I)^{-1/2} (L_{K,D_j}+\lambda I)^{-1/2} (L_{K,D_j}+\lambda'
  I)^{-1}\|\mathcal Q_{D_j,\lambda}\mathcal P_{D_j,\lambda}\\
  &+&
  |\lambda-\lambda'|\|(L_{K,D_j}+\mu_j  I)^{-1/2} (L_{K,D_j}+\lambda I)^{-1}(L_{K,D_j}+\lambda'I)^{-1}L_{K,D_j}L_K^{r-1/2}\| \|h_\rho\|_{\rho}\\
  &\leq&
   \mu_j ^{-1/2}\frac{|\lambda-\lambda'|}{\lambda'\sqrt{\lambda}}\mathcal Q_{D_j,\lambda}\mathcal P_{D_j,\lambda}
  +
   \mathcal Q_{D_j,\mu_j }^{2r-1}\frac{|\lambda-\lambda'|}{\lambda} \mu_j ^{r-1}\|h_\rho\|_\rho.
\end{eqnarray*}
 This completes the proof of Lemma \ref{Lemma:difference-krr}.
$\Box$
\end{proof}

For any $j=1,\dots,m$,  let 
$ T_{\Xi_L,j}:  B_{L,K}\rightarrow \mathbb R^{L}$ be the sampling  operator defined by  $T_{\Xi_L,j}f:=\{f(x_{i,j})\}_{j=1}^{|D_j|}$ such that the range of its adjoint  operator $ T^T_{\Xi_L,j}$ is exactly in $B_{L,K}$. Denote $L_{K,\Xi_L}=T_{\Xi_L,j}^TT_{\Xi_L,j}$ for convenience if no confusion concerning $j$ is made.
Write $
       U_j\Sigma_j V_j^T
$
as the SVD of $  T_{\Xi_L,j}$. Then, it is obvious that $V_j^TV_j=I$ and $P_{\Xi_L,j}=V_jV_j^T$ is the 
  corresponding
   projection operator from $\mathcal H_K$ to $ 
B_{L,K}$, implying
\begin{equation}\label{Projection-property}
    (I-P_{\Xi_L,j})^u=(I-P_{\Xi_L,j}),\qquad\forall u\in\mathbb N.
\end{equation}
Writing 
\begin{equation}\label{spetral definetion}
     h_{\Xi_L,\mu_j ,j}(L_{K,D_j}):=V_j(V_j^TL_{K,D_j}V_j+\mu_j 
     I)^{-1}V_j^T,
\end{equation} 
we obtain  from  \cite[eqs.(39)]{wang2023adaptive} or \cite{rudi2015less} 
that
\begin{equation}\label{Nystrom operator}
    g^{loc}_{D_j,\lambda_k,\mu_j }=h_{\Xi_L,\mu_j ,j}(L_{K,D_j})L_{K,D_j}g_{D_j,\lambda_k}.
\end{equation}
The property of $h_{\Xi_L,\mu_j ,j}(L_{K,D_j})$, exhibited in the following lemma, plays an important role in quantifying the performance of the local approximation. 
\begin{lemma}\label{Lemma:spectral-property}
Let $h_{\Xi_L,\mu_j ,j}(L_{K,D_j})$ be given in \eqref{spetral definetion}.  For an arbitrary bounded linear operator $B$, there holds
\begin{equation}\label{spectral-important-1}
     h_{\Xi_L,\mu_j ,j}(L_{K,D_j})(L_{K,D_j}+\mu_j 
     I)V_jBV_j^T
      =
     V_jBV_j^T.
\end{equation}
Moreover, for any $u,v\in[0,1/2]$   there holds
\begin{equation}\label{spectral-important-2}
    \|(L_{K,D_j}+\mu_j 
      I)^{u}h_{\Xi_L,\mu_j ,j}(L_{K,D_j})(L_{K,D_j}+\mu_j 
      I)^{v}\|\leq  \mu_j ^{-1+u+v}. 
\end{equation}
Furthermore,
\begin{equation}\label{spectral-important-3}
    \|h_{\Xi_L,\mu_j ,j}(L_{K,D_j})(L_{K,D_j}+\mu_j  I)\|\leq 1+\mathcal Q_{D_j,\mu_j }\mathcal Q_{\Xi_L,\mu_j }^*,
 \end{equation}
 where for any set of points $\Xi$,
\begin{equation}\label{Def:Q*}
    \mathcal Q^*_{\Xi,\mu_j }:=\|(L_{K}+\mu_j  I)^{-1/2}(L_{K,\Xi}+\mu_j 
        I)^{1/2}\|.
\end{equation}
\end{lemma}

\begin{proof}
Due to \eqref{spetral definetion} and $V_j^TV_j=I$, there holds
\begin{eqnarray*}
    &&  h_{\Xi_L,\mu_j ,j}(L_{K,D_j})(L_{K,D_j}+\mu_j 
     I)V_jBV_j^T
     =V_j(V_j^TL_{K,D_j}V_j+\mu_j 
     I)^{-1}V_j^T(L_{K,D_j}+\mu_j 
     I)V_jBV_j^T\\
     &=&
     V_j(V_j^TL_{K,D_j}V_j+\mu_j 
     I)^{-1}(V_j^TL_{K,D_j}V_j+\mu_j 
     I)BV_j^T
     =V_jBV_j^T,
\end{eqnarray*}
which proves \eqref{spectral-important-1}.
Using \eqref{spetral definetion} again, we get from Lemma \ref{Lemma:spectral-property}  with
$V_{HS}=V_j$, $A=L_{K,D_j}+\mu_j  I$, $B=V_j^T(L_{K,D_j}+\mu_j  I)^{-1}V_{j}$, $r^*=u$ and $s^*=v$ that 
\begin{eqnarray*}
       &&\|(L_{K,D_j}+\mu_j 
      I)^{u}h_{\Xi_L,\mu_j ,j}(L_{K,D_j})(L_{K,D_j}+\mu_j 
      I)^{v}\|=
      \|(L_{K,D_j}+\mu_j 
      I)^{u}V_j(V_j^TL_{K,D_j}V_j+\mu_j 
     I)^{-1}V_j^T(L_{K,D_j}+\mu_j 
      I)^{v}\|\\
      &\leq&
      \|(V_j^TL_{K,D_j}V_j+\mu_j 
     I)\|^{-1+u+v}
     \leq \mu_j ^{-1+u+v},
\end{eqnarray*}
where the last inequality is due to the fact that $V_j^TL_{K,D_j}V_j$ is semi-positive definite. 
This proves \eqref{spectral-important-2}. For any $f\in\mathcal H_K$, the triangle inequality yields
\begin{eqnarray*}
     \|h_{\Xi_L,\mu_j ,j}(L_{K,D_j})(L_{K,D_j}+\mu_j  I)\|
    \leq 
    \|h_{\Xi_L,\mu_j ,j}(L_{K,D_j})(L_{K,D_j}+\mu_j  I)P_{\Xi_L,j}\|+\|h_{\Xi_L,\mu_j ,j}(L_{K,D_j})(L_{K,D_j}+\mu_j  I)(I-P_{\Xi_L,j})\|. 
\end{eqnarray*}
Setting $B=I$ in \eqref{spectral-important-1}, we have
$$
     h_{\Xi_L,\mu_j ,j}(L_{K,D_j})(L_{K,D_j}+\mu_j 
     I)P_{\Xi_L,j}
     =P_{\Xi_L,j}. 
$$ 
Then, the definition of $P_{\Xi_L,j}$ yields
$$
   \|h_{\Xi_L,\mu_j ,j}(L_{K,D_j})(L_{K,D_j}+\mu_j  I)P_{\Xi_L,j}\|
   = 
   \|P_{\Xi_L,j}\|\leq 1.
$$
Since the range of $L_{K,\Xi_{L}}$ is exactly $\mathcal B_{\Xi,L}$. 
Then, we get from \eqref{spectral-important-2} with $u=v=1/2$, \eqref{Def.QD}, \eqref{Def:Q*} and 
Lemma \ref{Lemma:Projection general}    with $P=P_{\Xi_L,j}$ and $Z=T_{\Xi_L,j}$   that  
\begin{eqnarray*}
   &&\|h_{\Xi_L,\mu_j ,j}(L_{K,D_j})(L_{K,D_j}+\mu_j  I)(I-P_{\Xi_L,j})\|\\
   &\leq&
   \mu_j ^{-1/2}\|(L_{K,D_j}+\mu_j  I)^{1/2}h_{\Xi_L,\mu_j ,j}(L_{K,D_j})(L_{K,D_j}+\mu_j  I)^{1/2}\|
   \|(L_{K,D_j}+\mu_j  I)^{1/2}(I-P_{\Xi_L,j})\|\\
   &\leq&
   \|(L_{K,D_j}+\mu_j  I)^{1/2}(L_{K,\Xi_L}+\mu_j  I)^{-1/2}\|
   \leq 
   \|(L_{K,D_j}+\mu_j  I)^{1/2}(L_{K}+\mu_j  I)^{-1/2}\|
   \|(L_{K}+\mu_j  I)^{-1/2}(L_{K,\Xi_L}+\mu_j  I)^{-1/2}\|\\
   &\leq&
   \mathcal Q_{D_j,\mu_j }\mathcal Q_{\Xi_L,\mu_j }^*.
\end{eqnarray*}
Therefore, \eqref{spectral-important-3} holds and the 
  proof of Lemma \ref{Lemma:spectral-property} is finished. 
  $\Box$  
\end{proof}

As shown above, the bounds for operator differences $\mathcal P_{D_j,\lambda_j},\mathcal Q_{D_j,\lambda_j}$ and $\mathcal Q_{D_j,\lambda_j}^*$ are crucial in deducing generalization errors for KRR. The following lemma provides some basic results concerning these differences.

\begin{lemma}\label{Lemma:Q-1111}
Let $D$ be a sample drawn independently according to some distribution $\rho^*$, $\lambda>0$ and $0<\delta <1$.  
Under Assumption \ref{Assumption:boundedness}, if $\lambda\geq  \left[\frac{16(C_1^*)^2\log^3(16|D|)}{|D|}\right]$, then
with confidence $1-\delta$ there holds
\begin{eqnarray}
      \mathcal P_{D,\lambda}    &\leq & 
      2(\kappa M +\gamma) \left(\frac{1}{|D|\sqrt{\lambda}}+\frac{\sqrt{\mathcal
        N(\lambda)}}{\sqrt{|D|}}\right) \log
               \bigl(2/\delta\bigr) \label{bound-p-population}\\
               &\leq&
              3(\kappa M +\gamma) \left(\frac{1}{|D|\sqrt{ \lambda}}+\left( 1+ \frac{1}{\sqrt{\lambda|D|}}\right)
       \sqrt{\frac{\max\{\mathcal{N}_D(\lambda),1}{|D|}}\right)\log^2\frac8\delta, \label{bound-p}
\end{eqnarray}
and
\begin{equation}\label{Relation-effection}
    (1+4\eta_{\delta/4})^{-1}
		\sqrt{\max\{\mathcal N(\lambda),1\}}
    \leq 
		\sqrt{\max\{\mathcal N_D(\lambda),1\}} 
    \leq  (1+4\sqrt{\eta_{\delta/4}}\vee\eta_{\delta/4}^2)\sqrt{\max\{\mathcal N(\lambda),1\}}, 
\end{equation}
 where $\eta_\delta:=2\log(4/\delta)/\sqrt{\lambda|D|}$.
If in addition   
\begin{equation}\label{Res.delta}
    C_1^*\sqrt{\frac{  \log \left(1+8 \frac{1}{\sqrt{\lambda|D|}}\max\{1,\mathcal N_D(\lambda)\}\right)}{\lambda|D|}}\log\frac{16}{\delta}=:\mathcal B_{D,\lambda}\log\frac{16}{\delta}\leq \frac14, 
\end{equation}  
then with confidence at least
	$1-\delta$, there  holds
\begin{equation}\label{bound-q}
     \mathcal Q_{D,\lambda} \leq  \sqrt{2} ,
\end{equation}
\begin{equation}\label{bound-r}
    \|(L_K+\lambda I)^{-1/2}(L_K-L_{K,D})(L_K+\lambda I)^{-1/2}\|\leq \frac14,
\end{equation}
and
\begin{equation}\label{bound-q*}
    \mathcal Q_{D,\lambda}^*\leq \sqrt{2}.
\end{equation}
\end{lemma}

\begin{proof}
 The estimates in \eqref{bound-p-population} \eqref{bound-p} are standard in the learning theory literature which can be found in \cite{lin2018distributed,blanchard2016convergence} and \cite{lin2024adaptive,lu2020balancing}, respectively. The estimates in \eqref{Relation-effection} was derived in \cite{blanchard2019lepskii} and can also be found in \cite{lu2020balancing,lin2024adaptive}.
 The estimate in \eqref{bound-q} can be found in \cite{lin2020distributed,lin2024adaptive}. The estimate in \eqref{bound-r} can be easily derive by combining \cite[Lemma 6]{lin2020distributed} and \cite[Lemma 5.1]{lin2024adaptive}.  
 The  only thing remainder is to prove  \eqref{bound-q*}.    
Since \eqref{bound-q}, \eqref{bound-r}  and \eqref{Def.QD} yield
\begin{eqnarray*}
    \|(L_{K,D}+\lambda I)^{-1/2}(L_K-L_{K,D})(L_K+\lambda I)^{-1/2}\|
     \leq 
    \mathcal Q_{D,\lambda}^2 \|(L_K+\lambda I)^{-1/2}(L_K-L_{K,D})(L_K+\lambda I)^{-1/2}\|
    \leq \frac12,
\end{eqnarray*}
we have from $A^{-1}-B^{-1}=A^{-1}(B-A)B^{-1}$ with $A=(L_{K}+\lambda I)^{-1}$ and $B=(L_{K,D}+\lambda I)^{-1}$ that 
\begin{eqnarray*}
    &&
    \|(L_{K,D}+\lambda I)^{1/2}(L_{K}+\lambda I)^{-1}(L_{K,D}+\lambda I)^{1/2}\|\\
  &\leq  &
  1+\|(L_{K,D}+\lambda I)^{-1/2}(L_{K,D}-L_K)(L_{K,D}+\lambda I)^{-1/2}\| \|(L_{K,D}+\lambda I)^{1/2}(L_{K}+\lambda I)^{-1}(L_{K,D}+\lambda I)^{1/2}\|\\
  &\leq&
  1+\frac12\|(L_{K,D}+\lambda I)^{1/2}(L_{K}+\lambda I)^{-1}(L_K+\lambda I)^{1/2}\|.
\end{eqnarray*}
Therefore,
$$
  \|(L_{K,D}+\lambda I)^{1/2}(L_{K}+\lambda I)^{-1}(L_{K,D}+\lambda I)^{1/2}\|\leq 2,
$$
which follows \eqref{bound-q*}. This completes the proof of Lemma \ref{Lemma:Q-1111}.
 $\Box$
 \end{proof}


\subsection{Generalization error analysis on under-estimated agents}
The derivation of the generalization error on under-estimated agents is divided into four steps: bounds for local approximation, bounds for global approximation, bounds for $\overline{\mathcal W}_{D,\lambda_{k^*_j}}$ and generalization error estimates, which are collected in four propositions. 
The following proposition devotes to the first step.

\begin{proposition}\label{Proposition:noise-free-differences-123}
 If Assumption \ref{Assumption:regularity} holds with $1/2\leq r\leq 1$, then
\begin{eqnarray*}
   \|(L_{K,D_j}+\lambda_k I)^{1/2}g^{loc,\diamond}_{D_j,\lambda_k,\mu_j }\|_K
   &\leq&
 b\|h_\rho\|_\rho(  \mathcal Q_{D_j,\lambda_k}^{2r-1}+\mathcal Q_{D_j,\mu_j }\mathcal Q_{\Xi_L,\mu_j }^*+ 1)\lambda_{k-1}\lambda_k^{r}, \\
  \|(L_{K,D_j}+\lambda_k I)^{1/2}g^{loc}_{D_j,\lambda_k,\mu_j }\|_K
   &\leq&
    b\lambda_{k-1}(  \mathcal Q_{D_j,\lambda_k}^{2r-1}+\mathcal Q_{D_j,\mu_j }\mathcal Q_{\Xi_L,\mu_j }^*+ 2)(\|h_\rho\|_\rho\lambda_k^{r}+\mathcal Q_{D_j,\lambda_k} \mathcal P_{D_j,\lambda_k}),
\end{eqnarray*}
where  $ g^{loc,\diamond}_{D_j,\lambda_k,\mu_j } 
  :=h_{\Xi_L,\mu_j ,j}(L_{K,D_j})L_{K,D_j}g^\diamond_{D_j,\lambda_k}$ and $g^\diamond_{D_j,\lambda_k}:=f^\diamond_{D_j,\lambda_k}-f^\diamond_{D_j,\lambda_{k-1}}$.
\end{proposition}

\begin{proof}
Due to \eqref{KRR-local:operator} and \eqref{Different-abc}, there holds
$$
    g^\diamond_{D_j,\lambda_k}=
    (\lambda_{k-1}-\lambda_{k}) (L_{K,D_j}+\lambda_{k-1} I)^{-1}(L_{K,D_j}+\lambda_k I)^{-1}L_{K,D_j}f_\rho,
$$ 
which follows
\begin{eqnarray*}
    g^{loc,\diamond}_{D_j,\lambda_k,\mu_j } 
  =
  (\lambda_{k-1}-\lambda_{k})h_{\Xi_L,\mu_j ,j}(L_{K,D_j})L_{K,D_j}(L_{K,D_j}+\lambda_k I)^{-1}(L_{K,D_j}+\lambda_{k-1} I)^{-1}L_{K,D_j}f_\rho.
\end{eqnarray*}
 Therefore, 
\begin{eqnarray*}
    &&\| (L_{K,D_j}+\lambda_k I)^{1/2}g^{loc,\diamond}_{D_j,\lambda_k,\mu_j }\|_K
    \leq
   \|(L_{K,D_j}^{1/2}g^{loc,\diamond}_{D_j,\lambda_k,\mu_j }\|_K +\lambda_k^{1/2}\| g^{loc,\diamond}_{D_j,\lambda_k,\mu_j }\|_K\\
   &\leq&
  |\lambda_k-\lambda_{k-1}| \left\|  L_{K,D_j}+\mu_j  I)^{1/2}h_{\Xi_L,\mu_j ,j}(L_{K,D_j})L_{K,D_j}(L_{K,D_j}+\lambda_j I)^{-1}(L_{K,D_j}+\lambda_{k-1} I)^{-1}L_{K,D_j}f_\rho\right\|_K\\
  &+&
   |\lambda_k-\lambda_{k-1}|\lambda_k^{1/2} \left\| h_{\Xi_L,\mu_j ,j}(L_{K,D_j})L_{K,D_j}(L_{K,D_j}+\lambda_j I)^{-1}(L_{K,D_j}+\lambda_{k-1} I)^{-1}L_{K,D_j}f_\rho\right\|_K.
\end{eqnarray*}
Recalling \eqref{spectral-important-2} with $u=v=\frac12$ and \eqref{spectral-important-3}, we get
from 
 Assumption \ref{Assumption:regularity} with $\frac12\leq r\leq 1$ and   Lemma \ref{Lemma:cordes-11}   that
\begin{eqnarray*}
    &&\left\|(L_{K,D_j}+\mu_j  I)^{1/2}h_{\Xi_L,\mu_j ,j}(L_{K,D_j})L_{K,D_j}(L_{K,D_j}+\lambda_k I)^{-1}(L_{K,D_j}+\lambda_{k-1} I)^{-1}L_{K,D_j}f_\rho\right\|_K\\
    &\leq &
    \|(L_{K,D_j}+\mu_j  I)^{1/2}h_{\Xi_L,\mu_j ,j}(L_{K,D_j})(L_{K,D_j}+\mu_j  I)^{1/2}\|\\
    &\times&
    \|L_{K,D_j}^{1/2}(L_{K,D_j}+\lambda_k I)^{-1}(L_{K,D_j}+\lambda_{k-1} I)^{-1}L_{K,D_j}L_K^{r-1/2}\|_K\|h_\rho\|_\rho\\
     &\leq&
     \mathcal Q_{D_j,\lambda_k}^{2r-1}   \lambda_k^{r-1}\|h_\rho\|_\rho 
\end{eqnarray*}
and
\begin{eqnarray*}
    &&\left\| h_{\Xi_L,\mu_j ,j}(L_{K,D_j})L_{K,D_j}(L_{K,D_j}+\lambda_j I)^{-1}(L_{K,D_j}+\lambda_{k-1} I)^{-1}L_{K,D_j}f_\rho\right\|_K\\
    &\leq&
    \left\| h_{\Xi_L,\mu_j ,j}(L_{K,D_j})L_{K,D_j}\right\|
    \left\|(L_{K,D_j}+\lambda_j I)^{-1}(L_{K,D_j}+\lambda_{k-1} I)^{-1}L_{K,D_j}L_K^{r-1/2}\right\|\|h_\rho\|_\rho \\
    &\leq&
    (1+\mathcal Q_{D_j,\mu_j }\mathcal Q_{\Xi_L,\mu_j }^*)\|h_\rho\|_\rho\lambda_k^{r-3/2}.
\end{eqnarray*}
Hence,
\begin{eqnarray*}
    &&\|(L_{K,D_j}+\lambda_k I)^{1/2}g^{loc,\diamond}_{D_j,\lambda_k,\mu_j }\|_K
    \leq
    \frac{\lambda_{k-1}-\lambda_k}{\lambda_k}\|h_\rho\|_\rho(  \mathcal Q_{D,\lambda_k}^{2r-1}+\mathcal Q_{D_j,\mu_j }\mathcal Q_{\Xi_L,\mu_j }^*+ 1)\lambda_k^{r}.  
\end{eqnarray*}
Noting $\lambda_k=\frac1{kb}$, we have
\begin{equation}\label{lambda-com}
     \frac{|\lambda_k-\lambda_{k-1}|}{\lambda_k}=b\lambda_{k-1},\qquad \frac{|\lambda_k-\lambda_{k}|}{\lambda_k}=b\lambda_{k}.
\end{equation}
This naturally follows  
\begin{eqnarray*}
    &&\|(L_{K,D_j}+\lambda_k I)^{1/2}g^{loc,\diamond}_{D_j,\lambda_k,\mu_j }\|_K
    \leq
    b\lambda_{k-1}\|h_\rho\|_\rho(  \mathcal Q_{D_j,\lambda_k}^{2r-1}+\mathcal Q_{D_j,\mu_j }\mathcal Q_{\Xi_L,\mu_j }^*+ 1)\lambda_k^{r}.  
\end{eqnarray*}
Based on  definitions of $g^{loc}_{D_j,\lambda_{k},\mu_j }$ and $g^{loc,\diamond}_{D_j,\lambda_{k},\mu_j }$, it is easy to derive
$$
      g^{loc}_{D_j,\lambda_{k},\mu_j }
        - 
       g^{loc,\diamond}_{D_j,\lambda_{k},\mu_j } 
       =
       (\lambda_{k-1}-\lambda_{k})h_{\Xi_L,\mu_j ,j}(L_{K,D_j})L_{K,D_j}(L_{K,D_j}+\lambda_{k} I)^{-1} (L_{K,D_j}+\lambda_{k-1}
  I)^{-1}(S_{D_j}^Ty_{D_j}-L_{K,D_j}f_\rho) .
$$
We then have
\begin{eqnarray*}
    && \|(L_{K,D_j}+\lambda_k I)^{1/2}g^{loc}_{D_j,\lambda_k,\mu_j }\|_K\leq \|(L_{K,D_j}+\lambda_k I)^{1/2}g^{loc,\diamond}_{D_j,\lambda_k,\mu_j }\|_K
    +\|(L_{K,D_j}+\lambda_k I)^{1/2}(g^{loc}_{D_j,\lambda_k,\mu_j }-g^{loc,\diamond}_{D_j,\lambda_k,\mu_j })\|_K\\
    &\leq&
    b\lambda_{k-1}\|h_\rho\|_\rho(  \mathcal Q_{D_j,\lambda_k}^{2r-1}+\mathcal Q_{D_j,\mu_j }\mathcal Q_{\Xi_L,\mu_j }^*+ 1)\lambda_k^{r}\\
    &+&
    (\lambda_{k-1}-\lambda_{k})\|(L_{K,D_j}+\lambda_k I)^{1/2}h_{\Xi_L,\mu_j ,j}(L_{K,D_j})L_{K,D_j}(L_{K,D_j}+\lambda_{k} I)^{-1} (L_{K,D_j}+\lambda_{k-1}
  I)^{-1}(S_{D_j}^Ty_{D_j}-L_{K,D_j}f_\rho)\|_K.
\end{eqnarray*}
 But \eqref{spectral-important-2} with $u=v=1/2$, \eqref{spectral-important-3}, \eqref{Def.QD} and \eqref{Def.PD} yield
 \begin{eqnarray*}
     &&
     \|(L_{K,D_j}+\lambda_k I)^{1/2}h_{\Xi_L,\mu_j ,j}(L_{K,D_j})L_{K,D_j}(L_{K,D_j}+\lambda_{k} I)^{-1} (L_{K,D_j}+\lambda_{k-1}
  I)^{-1}(S_{D_j}^Ty_{D_j}-L_{K,D_j}f_\rho)\|_K\\
  &\leq&
  \|L_{K,D_j}^{1/2}h_{\Xi_L,\mu_j ,j}(L_{K,D_j})L_{K,D_j}(L_{K,D_j}+\lambda_{k} I)^{-1} (L_{K,D_j}+\lambda_{k-1}
  I)^{-1}(S_{D_j}^Ty_{D_j}-L_{K,D_j}f_\rho)\|_K\\
  &+&
  \lambda_k^{1/2}\|h_{\Xi_L,\mu_j ,j}(L_{K,D_j})L_{K,D_j}(L_{K,D_j}+\lambda_{k} I)^{-1} (L_{K,D_j}+\lambda_{k-1}
  I)^{-1}(S_{D_j}^Ty_{D_j}-L_{K,D_j}f_\rho)\|_K\\
  &\leq&
  \|L_{K,D_j}^{1/2}(L_{K,D_j}+\lambda_{k-1}
  I)^{-1}(L_{K,D_j}+\lambda_{k} I)^{-1/2}\|\mathcal Q_{D_j,\lambda_k} \mathcal P_{D_j,\lambda_k}\\
  &+&
  (1+\mathcal Q_{D_j,\mu_j }\mathcal Q_{\Xi_L,\mu_j }^*)\lambda_k^{1/2}\|  (L_{K,D_j}+\lambda_{k-1}
  I)^{-1}(L_{K,D_j}+\lambda_{k} I)^{-1/2}\|\mathcal Q_{D_j,\lambda_k} \mathcal P_{D_j,\lambda_k}\\
  &\leq&
  (2+\mathcal Q_{D_j,\mu_j }\mathcal Q_{\Xi_L,\mu_j }^*)\lambda_{k-1}^{-1}\mathcal Q_{D_j,\lambda_k} \mathcal P_{D_j,\lambda_k}.
 \end{eqnarray*}
Hence,  \eqref{lambda-com} implies
$$
  \|(L_{K,D_j}+\lambda_k I)^{1/2}g^{loc}_{D_j,\lambda_k,\mu_j }\|_K 
    \leq
    b\lambda_{k-1}\|h_\rho\|_\rho(  \mathcal Q_{D_j,\lambda_k}^{2r-1}+\mathcal Q_{D_j,\mu_j }\mathcal Q_{\Xi_L,\mu_j }^*+ 1)\lambda_k^{r}+b\lambda_{k}(2+\mathcal Q_{D_j,\mu_j }\mathcal Q_{\Xi_L,\mu_j }^*) \mathcal Q_{D_j,\lambda_k} \mathcal P_{D_j,\lambda_k}.
$$
  This
completes the proof of Proposition \ref{Proposition:noise-free-differences-123}.
$\Box$
\end{proof}

To derive bounds for global approximation, we need the following lemma to show    the power of synthesization.
\begin{lemma}\label{Lemma:power-synthesization}
For any  $g_j\in\mathcal H_K$, $j=1,\dots,m$  there holds
$$
     E\left[\left\|\sum_{j=1}^m\frac{|D_j|}{|D|}g_j\right\|_K^2\right]\leq 
     \sum_{j=1}^m\frac{|D_j|^2}{|D|^2} E\left[\|g_j\|_K^2\right]
    +
    \left\|\sum_{j=1}^m\frac{|D_j|}{|D|}E[g_j]\right\|_K^2.
$$ 
\end{lemma}

\begin{proof}
Since $\sum_{j=1}^m\frac{|D_j|}{|D|}=1$, we have
\begin{eqnarray*}
    \left\|\sum_{j=1}^m\frac{|D_j|}{|D|}g_j\right\|_K^2
    =
   \sum_{j=1}^m\frac{|D_j|^2}{|D|^2} \|g_j\|_K^2
    +
    \sum_{j=1}^m\frac{|D_j|}{|D|}
    \left\langle g_j,\sum_{k\neq j}\frac{|D_k|}{|D|}g_k\right\rangle_K.
\end{eqnarray*}
Taking expectations on both sides, we then get
\begin{eqnarray*}
   && E\left[\left\|\sum_{j=1}^m\frac{|D_j|}{|D|}g_j\right\|_K^2\right]
   =
   \sum_{j=1}^m\frac{|D_j|^2}{|D|^2} E\left[\|g_j\|_K^2\right]
   +
   \sum_{j=1}^m\frac{|D_j|}{|D|}
    \left\langle E_{D_j}[g_j],\sum_{k\neq j}\frac{|D_k|}{|D|}E_{D_k}[g_k]\right\rangle_K\\
    &=&
    \sum_{j=1}^m\frac{|D_j|^2}{|D|^2} E\left[\|g_j\|_K^2\right]
    +
    \sum_{j=1}^m\frac{|D_j|}{|D|}\left\langle E_{D_j}[g_j],\sum_{j=1}^m\frac{|D_j|}{|D|}E_{D_j}[g_j]-\frac{|D_j|}{|D|}E_{D_j}[g_j] 
    \right\rangle_K\\
    &=&
    \sum_{j=1}^m\frac{|D_j|^2}{|D|^2} E\left[\|g_j\|_K^2\right]
    +
    \left\|\sum_{j=1}^m\frac{|D_j|}{|D|}E_{D_j}[g_j]\right\|_K^2
    -
    \sum_{j=1}^m\frac{|D_j|^2}{|D|^2}\|E_{D_j}[g_j]\|_K^2\\
    &\leq&
    \sum_{j=1}^m\frac{|D_j|^2}{|D|^2} E\left[\|g_j\|_K^2\right]
    +
    \left\|\sum_{j=1}^m\frac{|D_j|}{|D|}E_{D_j}[g_j]\right\|_K^2,
\end{eqnarray*}
where $E_{D_j}$ denotes the expectation with respect to $D_j$. 
This completes the proof of Lemma \ref{Lemma:power-synthesization}.
$\Box$
\end{proof}

By the help  of the Proposition \ref{Proposition:noise-free-differences-123} and Lemma \ref{Lemma:power-synthesization}, 
we   present the following bound for global approximation.
\begin{proposition}\label{Prop:global-approximation}
  Under Assumption \ref{Assumption:regularity}  with $1/2\leq r\leq 1$,  for any $j=1,\dots,m$, there holds
\begin{eqnarray*}
    &&E\left[ \|(L_{K,D_j}+\lambda_k I)^{1/2} g^{global}_{D,\lambda_k,\mu}\|_{K}^2
    \right]\\
      &\leq&  
     2 b^2  \lambda_{k-1}^2 \sum_{j'=1}^m\frac{|D_{j'}|^2}{|D|^2}E[(\mathcal Q^*_{D_j,\lambda_k })^2\mathcal Q^4_{D_{j'},\lambda_k } (  \mathcal Q_{D_{j'},\lambda_k}^{2r-1}+\mathcal Q_{D_{j'},\mu_j }\mathcal Q_{\Xi_L,\mu_j }^*+ 2)^2   \mathcal P_{D_{j'},\lambda_k}^2]    \\
    &+&
    3b^2\lambda^2_{k-1}\lambda_k^{2r}\|h_\rho\|_\rho^2\sum_{j'=1}^m\frac{|D_{j'}|}{|D|}
 E[(\mathcal Q^*_{D_j,\lambda_k })^2\mathcal Q^2_{D_{j'},\lambda_k } (  \mathcal Q_{D_{j'},\lambda_k}^{2r-1}+\mathcal Q_{D_{j'},\mu_j }\mathcal Q_{\Xi_L,\mu_j }^*+ 2)^2  ].
\end{eqnarray*}
\end{proposition}

\begin{proof}
For any $j=1,\dots,m$,  \eqref{global-approximation-1} implies
$$
   \|(L_{K,D_j}+\lambda_k I)^{1/2} g^{global}_{D,\lambda_k,\mu}\|_K
  =
   \left\|(L_{K,D_j}+\lambda_k I)^{1/2}\sum_{j'=1}^m\frac{|D_{j'}|}{|D|}  g^{loc}_{D_{j'},\lambda_k,\mu_j }\right\|_K.
$$
It then follows from Lemma \ref{Lemma:power-synthesization} with $g_{j'}=  (L_{K,D_j}+\lambda_k I)^{1/2}g^{loc}_{D_{j'},\lambda_k,\mu_j }$, Jensen's inequality  and  $E[g_{D_j,\lambda_k,\mu_j }^{loc}|x]=g_{D_j,\lambda_k,\mu_j }^{loc,\diamond}$   that
\begin{eqnarray*}
   &&E\left[\|(L_{K,D_j}+\lambda_k I)^{1/2}g^{global}_{D,
   \lambda_k,\mu }\|_{K}^2\right]\\
    &\leq&
    \sum_{j'=1}^m\frac{|D_{j'}|}{|D|}E\left[\frac{|D_{j'}|}{|D|}\|(L_{K,D_j}+\lambda_k I)^{1/2}g_{D_{j'},\lambda_k,\mu_j }^{loc}\|_K^2
    +E[\|(L_{K,D_j}+\lambda_k I)^{1/2}g_{D_{j'},\lambda_k,\mu_j }^{loc,\diamond}\|^2_K\right].
\end{eqnarray*}
Recalling  
$$
   \|(L_{K,D_j}+\lambda_k I)^{1/2}f\|_K\leq \mathcal Q_{D_j,\lambda_k }^*\mathcal Q_{D_{j'},\lambda_k }\|(L_{K,D_{j'}}+\lambda_k  I)^{1/2}f\|_K,\qquad\forall f\in\mathcal H_K,
$$
we have from Proposition \ref{Proposition:noise-free-differences-123} that
\begin{eqnarray*}
   && \sum_{j'=1}^m\frac{|D_{j'}|}{|D|}E[\|(L_{K,D_{j}}+\lambda_k I)^{1/2}g_{D_j,\lambda_k,\mu_j }^{loc,\diamond}\|_K^2]
   \leq 
   \sum_{j'=1}^m\frac{|D_{j'}|}{|D|}E[(\mathcal Q_{D_j,\lambda_k }^*\mathcal Q_{D_{j'},\lambda_k })^2\|(L_{K,D_{j'}}+\lambda_k I)^{1/2}g_{D_j,\lambda_k,\mu_j }^{loc,\diamond}\|_K^2]\\
   &\leq&
 b^2\lambda^2_{k-1}\lambda_k^{2r}\|h_\rho\|_\rho^2\sum_{j'=1}^m\frac{|D_{j'}|}{|D|}
 E[(\mathcal Q^*_{D_j,\lambda_k })^2\mathcal Q^2_{D_{j'},\lambda_k } (  \mathcal Q_{D_{j'},\lambda_k}^{2r-1}+\mathcal Q_{D_{j'},\mu_j }\mathcal Q_{\Xi_L,\mu_j }^*+ 1)^2  ]
\end{eqnarray*}
and
\begin{eqnarray*}
      &&\sum_{j'=1}^m\frac{|D_{j'}|^2}{|D|^2}E[\|(L_{K,D_j}+\mu_j   I)^{1/2}g_{D_{j'},\lambda_k}^{loc}\|_K^2]
      \leq
      \sum_{j'=1}^m\frac{|D_{j'}|^2}{|D|^2}E[(\mathcal Q^*_{D_j,\lambda_k })^2\mathcal Q^2_{D_{j'},\lambda_k } \|(L_{K,D_{j'}}+\lambda_k  I)^{1/2}g_{D_{j'},\lambda_k}^{loc}\|_K^2]
      \\
   &\leq&
  b^2  \lambda_{k-1}^2 \sum_{j'=1}^m\frac{|D_{j'}|^2}{|D|^2}E[(\mathcal Q^*_{D_j,\lambda_k })^2\mathcal Q^2_{D_{j'},\lambda_k } (  \mathcal Q_{D_{j'},\lambda_k}^{2r-1}+\mathcal Q_{D_{j'},\mu_j }\mathcal Q_{\Xi_L,\mu_j }^*+ 2)^2 (\|h_\rho\|_\rho\lambda_k^{r}+\mathcal Q_{D_{j'},\lambda_k} \mathcal P_{D_{j'},\lambda_k})^2        ].
\end{eqnarray*}
Therefore,   
\begin{eqnarray*}
    &&E\left[ \|(L_{K,D_j}+\lambda_k I)^{1/2} g^{global}_{D,\lambda_k,\mu}\|_{K}^2
    \right]\\
      &\leq&  
     2 b^2  \lambda_{k-1}^2 \sum_{j'=1}^m\frac{|D_{j'}|^2}{|D|^2}E[(\mathcal Q^*_{D_j,\lambda_k })^2\mathcal Q^4_{D_{j'},\lambda_k } (  \mathcal Q_{D_{j'},\lambda_k}^{2r-1}+\mathcal Q_{D_{j'},\mu_j }\mathcal Q_{\Xi_L,\mu_j }^*+ 2)^2   \mathcal P_{D_{j'},\lambda_k}^2]    \\
    &+&
    3b^2\lambda^2_{k-1}\lambda_k^{2r}\|h_\rho\|_\rho^2\sum_{j'=1}^m\frac{|D_{j'}|}{|D|}
 E[(\mathcal Q^*_{D_j,\lambda_k })^2\mathcal Q^2_{D_{j'},\lambda_k } (  \mathcal Q_{D_{j'},\lambda_k}^{2r-1}+\mathcal Q_{D_{j'},\mu_j }\mathcal Q_{\Xi_L,\mu_j }^*+ 2)^2  ].
\end{eqnarray*}
This completes the proof of Proposition \ref{Prop:global-approximation}.
 $\Box$
\end{proof}
To continue the analysis, we need upper bounds of operator differences in expectation, requiring the following standard lemma.

\begin{lemma}\label{Lemma:prob-to-exp}
Let $0<\delta<1$ and $\xi\in\mathbb R_+$ be a random variable. If $\xi\leq a\log^b\frac{c}{\delta}$ holds with confidence $1-\delta$ for $\delta\geq\delta_0$   for some $a,b,c>0$ and $\delta_0>0$, then
$$
      E[\xi]\leq \delta_0+c\Gamma(b+1) a,
$$
where $\Gamma(\cdot)$ is the Gamma function.
\end{lemma}
\begin{proof}
Direct computation yields
\begin{eqnarray*}
     E[\Xi]=\int_{0}^\infty P[\xi>t]dt
   =\int_{0}^{\delta_0}P[\xi>t]dt
  +\int_{\delta_0}^{\infty}P[\xi>t]dt
  \leq \delta_0+\int_{\delta_0}^{\infty} c\exp\{-(t/a)^{-1/b}\}dt,
\end{eqnarray*}
where the last inequality is derived based on the fact $P[\xi<a\log^b\frac{c}{\delta}]\geq 1-\delta$ for any $\delta\geq \delta_0$.
Therefore, 
$$
    E[\Xi]\leq \delta_0+\int_{0}^{\infty} c\exp\{-(t/a)^{-1/b}\}dt\leq \delta_0+c\Gamma(b+1) a,
$$
which completes the proof of Lemma \ref{Lemma:prob-to-exp}. $\Box$
\end{proof}

Based on the above lemma and Lemma \ref{Lemma:Q-1111}, we derive the following 
 upper bounds for operator differences.

\begin{lemma}\label{Lemma:Q222-111}
Under \eqref{res.theorem} and Assumption \ref{Assumption:boundedness}, if $\Xi_L=\{\xi_\ell\}_{\ell=1}^L$ is a set of i.i.d. random variables with $
     L \geq \max_{j=1,\dots,m} |D_j|
$, $K^*$ is given  by \eqref{DefK*}, and $\mu_j  $ satisfies \eqref{Selection-mu}, then for any $k\leq K^*$, there holds
\begin{eqnarray*}
     E[(\mathcal Q^*_{D_j,\lambda_k })^2\mathcal Q^4_{D_{j'},\lambda_k } (  \mathcal Q_{D_{j'},\lambda_k}^{2r-1}+\mathcal Q_{D_{j'},\mu_j }\mathcal Q_{\Xi_L,\mu_j }^*+ 2)^2   \mathcal P_{D_{j'},\lambda_k}^2] 
     &\leq&
     (1+576\Gamma(5)(\kappa M +\gamma)^2(2^{r-1/2}+4)^2 )\mathcal W_{D_j,\lambda_k}^2,\\
     E[(\mathcal Q^*_{D_j,\lambda_k })^2\mathcal Q^2_{D_{j'},\lambda_k } (  \mathcal Q_{D_{j'},\lambda_k}^{2r-1}+\mathcal Q_{D_{j'},\mu_j }\mathcal Q_{\Xi_L,\mu_j }^*+ 2)^2  ]&\leq& 5(2\sqrt{2}+4)^2,\\
      E[\mathcal Q_{D_{j'},\lambda_{k_j^*}}^{4r}]
      &\leq& 2^{2r+1}, \\
      E[\mathcal Q_{D_{j'},\lambda_{k^*_j}}^4\mathcal P_{D_{j'},\lambda_{k^*_j}}^2]
      &\leq& 19(\kappa M+\gamma)^2\mathcal W_{D_j,\lambda_{k_j^*}}.
\end{eqnarray*}
\end{lemma}

\begin{proof} 
Due to Lemma \ref{Lemma:Q-1111}, with confidence $1-\delta$ for $\delta$ satisfying 
\begin{equation}\label{res.11334}
        \max\{\mathcal B_{D_j,\lambda_k},\mathcal B_{D_j,\mu_j},\mathcal B_{\Xi_L,\mu_j}\}\log\frac{16}{\delta}\leq \frac14 , 
\end{equation}
there holds 
\begin{eqnarray*}
   && (\mathcal Q^*_{D_j,\lambda_k })^2\mathcal Q^4_{D_{j'},\lambda_k } (  \mathcal Q_{D_{j'},\lambda_k}^{2r-1}+\mathcal Q_{D_{j'},\mu_j }\mathcal Q_{\Xi_L,\mu_j }^*+ 2)^2   \mathcal P_{D_{j'},\lambda_k}^2\\
    &\leq&
      72(\kappa M +\gamma)^2(2^{r-1/2}+4)^2 \left(\frac{1}{|D_j|\sqrt{ \lambda_k}}+\left( 1+ \frac{1}{\sqrt{\lambda_k|D_j|}}\right)
       \sqrt{\frac{\max\{\mathcal{N}_{D_j}(\lambda_k),1}{|D_j|}}\right)^2\log^4\frac{8}\delta,
\end{eqnarray*}
Since  \eqref{res.11334} implies
$$
     \delta\geq 16 \exp\left\{-4\max\{\mathcal B_{D_j,\lambda_k},\mathcal B_{D_j,\mu_j},\mathcal B_{\Xi_L,\mu_j}\}\right\}
$$
while  \eqref{DefK*} together with $\lambda_k=\frac1{kb}$ and \eqref{res.theorem}, \eqref{Selection-mu} and $L\geq\max_{j=1,\dots,m}\{|D_j|\}$ yields 
$$
   16 \exp\left\{-4\max\{\mathcal B_{D_j,\lambda_k},\mathcal B_{D_j,\mu_j},\mathcal B_{\Xi_L,\mu_j}\}\right\}\leq 1/|D|,\qquad,\forall k=1,\dots,K^*,j=1,\dots,m.
$$
We then have from Lemma \ref{Lemma:prob-to-exp} and \eqref{Def.WD}  that
\begin{eqnarray*}
    && E[(\mathcal Q^*_{D_j,\lambda_k })^2\mathcal Q^4_{D_{j'},\lambda_k } (  \mathcal Q_{D_{j'},\lambda_k}^{2r-1}+\mathcal Q_{D_{j'},\mu_j }\mathcal Q_{\Xi_L,\mu_j }^*+ 2)^2   \mathcal P_{D_{j'},\lambda_k}^2]\\  
    &\leq&
    216\exp\{-4\mathcal B_{D,\lambda}\}+576\Gamma(5)(\kappa M +\gamma)^2(2^{r-1/2}+4)^2 \mathcal W_{D_j,\lambda_k}^2
    \leq 
   (1+576\Gamma(5)(\kappa M +\gamma)^2(2^{r-1/2}+4)^2 )\mathcal W_{D_j,\lambda_k}^2.
\end{eqnarray*}
Furthermore, Lemma \ref{Lemma:Q-1111} also shows
$$
   (\mathcal Q^*_{D_j,\lambda_k })^2\mathcal Q^2_{D_{j'},\lambda_k } (  \mathcal Q_{D_{j'},\lambda_k}^{2r-1}+\mathcal Q_{D_{j'},\mu_j }\mathcal Q_{\Xi_L,\mu_j }^*+ 2)^2 
   \leq  4(2\sqrt{2}+4)^2\log^v\frac{1}{\delta}
$$
for any $v>0$. 
Then the same method as above derives 
$$ 
E[(\mathcal Q^*_{D_j,\lambda_k })^2\mathcal Q^2_{D_{j'},\lambda_k } (  \mathcal Q_{D_{j'},\lambda_k}^{2r-1}+\mathcal Q_{D_{j'},\mu_j }\mathcal Q_{\Xi_L,\mu_j }^*+ 2)^2  ]\leq 5(2\sqrt{2}+4)^2 
$$  
directly. The proof the remaining two terms are almost the same as above. We remove them for the sake of brevity.
This completes the 
proof of Lemma \ref{Lemma:Q222-111}. $\Box$ 
\end{proof}

Our third step, showing in the following proposition,  derives an upper bound for averaged effective dimension by using the above proposition and the
definition of $k^*_j$.

\begin{proposition}\label{Proposition:bound on effect}
Let $k^*$ be given in \eqref{stopping-1}. Under Assumption \ref{Assumption:boundedness} and Assumption \ref{Assumption:regularity}  with $1/2\leq r\leq 1$,  if $\Xi_L=\{\xi_\ell\}_{\ell=1}^L$ is a set of i.i.d. random variables with $
     L \geq \max_{j=1,\dots,m} |D_j|
$, \eqref{res.m} holds and $\mu_j $ satisfies \eqref{Selection-mu},
then 
\begin{equation}\label{Bound.w}
    \overline{\mathcal W}_{D,\lambda_{k^*_j}} \leq
    \tilde{C}_1 \lambda_{k^*_j}^{2r},
\end{equation}
where  $\tilde{C}_1:=30(2\sqrt{2}+4)^2 b^2\|h_\rho\|_\rho^2/C_{LP}.$
\end{proposition}

\begin{proof}
Due to \eqref{stopping-1}, there holds
$$
  \|(L_{K,D_j}+\lambda_{k_j^*} I)^{1/2}g^{global}_{D,\lambda_{k^*_j},\mu }\|_K^2 
    \geq
    C_{LP}  \lambda_{k^*_j-1}^2\overline{\mathcal W}_{D,\lambda_{k^*_j}}.
$$
Taking expectation on both side, we obtain
$$
     C_{LP}  \lambda_{k^*_j-1}^2\overline{\mathcal W}_{D,\lambda_{k^*_j}}
     \leq 
     E[\|(L_{K,D_j}+\lambda_{k_j^*} I)^{1/2}g^{global}_{D,\lambda_{k^*_j},\mu }\|_K^2].
$$
But Proposition \ref{Prop:global-approximation} together with Lemma \ref{Lemma:Q222-111} yields
\begin{eqnarray*}
    E\left[ \|(L_{K,D_j}+\lambda_{k_j^*} I)^{1/2} g^{global}_{D,\lambda_{k^*_j},\mu }\|_{K}^2
    \right]
     &\leq& 
     2 b^2 (1+576\Gamma(5)(\kappa M +\gamma)^2(\sqrt{2}+4)^2 ) \lambda_{k_{j}^*-1}^2 \sum_{j'=1}^m\frac{|D_{j'}|^2}{|D|^2}
     \mathcal W_{D_j,\lambda_{k_j^*}}^2 \\
   & +&
   15(2\sqrt{2}+4)^2 b^2\lambda^2_{k_j^*-1}\lambda_{k_j^*}^{2r}\|h_\rho\|_\rho^2.
\end{eqnarray*}
Hence, the definition of $C_{LP}$ yileds
$$
   C_{LP}  \lambda_{k^*_j-1}^2\overline{\mathcal W}_{D,\lambda_{k^*_j}}
   \leq 
   15(2\sqrt{2}+4)^2 b^2\lambda^2_{k_j^*-1}\lambda_{k_j^*}^{2r}\|h_\rho\|_\rho^2
      +
      \frac{C_{LP}}{2}
      \lambda_{k^*_j-1}^2\overline{\mathcal W}_{D,\lambda_{k^*_j}}.
$$
This means
$$
      \overline{\mathcal W}_{D,\lambda_{k^*_j}}
      \leq 
    30(2\sqrt{2}+4)^2 b^2\|h_\rho\|_\rho^2/C_{LP}
       \lambda_{k^*_j}^{2r}.
$$
and completes the proof of Proposition \ref{Proposition:bound on effect} with $\tilde{C}_1:=30(2\sqrt{2}+4)^2 b^2\|h_\rho\|_\rho^2/C_{LP}$.
 $\Box$   
\end{proof}

Based on the above proposition, we present  the following crucial result  for generalization error analysis.
\begin{proposition}\label{Prop:small-lambda}
Let $1\leq j\leq m$ and $\lambda_0=\lambda_{k_0}$ for $k_0=[|D|^{1/(2r+s)}]$.  
Under Assumption \ref{Assumption:boundedness}, Assumption \ref{Assumption:regularity}   with $\frac12\leq r\leq1$ and Assumption \ref{Assumption:effective dimension} with $0\leq s\leq 1$, if   $\Xi_L=\{\xi_\ell\}_{\ell=1}^L$ is a set of i.i.d. random variables with $
     L \geq \max_{j=1,\dots,m} |D_j|,
$
$\lambda_{k_j^*}\leq\lambda_0$, \eqref{res.m} holds, and $\mu_j $ satisfies \eqref{Selection-mu}, then
\begin{equation}\label{small-rho-norm}
     \max\{E[\|(\overline{f}_j-f_\rho)\|^2_\rho],E[\|(\overline{f}_j^\diamond-f_\rho)\|^2_\rho]\}
  \leq 
  \tilde{C}_2|D|^{-\frac{2r}{2r+s}}, 
\end{equation}
and
\begin{equation}\label{small-K-norm}
     \max\{E[\|(\overline{f}_j-f_\rho)\|^2_K],E[\|(\overline{f}_j^\diamond-f_\rho)\|^2_K]\} 
  \leq 
  \tilde{C}_2|D|^{-\frac{2r}{2r+s}},
\end{equation}
where
$\tilde{C}_2:=(16\|h\|_\rho^2+19(\kappa M+\gamma)^2\tilde{C}_1)$
and  
$ 
 \overline{f}_j^\diamond:= \sum_{j'=1}^m \frac{|D_{j'}|}{|D|}
      f^\diamond_{D_{j'},\lambda_j^*}.
$
\end{proposition}
  
\begin{proof}
It follows from Jensen's inequality, \eqref{bound.approximation-error} and  Lemma \ref{Lemma:Q222-111} that
$$
         E[\|(L_K+\lambda_{k_j^*} I)^{1/2}(\overline{f}_j^\diamond-f_\rho)\|_K^2]
         \leq
         \sum_{j'=1}^{m}\frac{|D_{j'}|}{|D|}E[\mathcal Q_{D_{j'},\lambda_{k_j^*}}^{4r}]\lambda_{k^*_j}^{2r}\|h_\rho\|^2_\rho
         \leq
         8\lambda_{k^*_j}^{2r}\|h_\rho\|^2_\rho.
$$
Moreover, we obtain from \eqref{bound.sample-error}, Lemma \ref{Lemma:Q222-111} and Proposition \ref{Proposition:bound on effect} that for any $j=1,\dots,m$, 
\begin{eqnarray*}
     &&\sum_{j'=1}^m\frac{|D_{j'}|^2}{|D|^2} E[\|(L_K+\lambda_{k_j^*} I)(f_{D_{j'},\lambda_{j}^*}-f_\rho)\|_K^2]  
    \leq
   \sum_{j'=1}^m\frac{|D_{j'}|^2}{|D|^2}E[\mathcal Q_{D_{j'},\lambda_{k^*_j}}^4\mathcal P_{D_{j'},\lambda_{k^*_j}}^2]\\
   &\leq&
   19(\kappa M+\gamma)^2\sum_{j'=1}^m\frac{|D_{j'}|^2}{|D|^2}\mathcal W_{D_{j'},\lambda_{k_j^*}}
   =19(\kappa M+\gamma)^2\overline{\mathcal W}_{D,\lambda_{k_j^*}}\\
   &\leq &
   19(\kappa M+\gamma)^2\tilde{C}_1\lambda_{k_j^*}^{2r}.
\end{eqnarray*}
Hence,
  Lemma \ref{Lemma:power-synthesization}   with $g_{j'}:=(L_K+\lambda_{k_j^*} I)(f_{D_{j'},\lambda_j^*}-f_\rho)$  shows
\begin{eqnarray*}
   &&E[\|(L_K+\lambda_{k_j^*} I)(\overline{f}_j-f_\rho)\|^2_K]\\
   &\leq&
   2\sum_{j'=1}^{m}\frac{|D_{j'}|}{|D|}E[\|(L_K+\lambda_{k_j^*} I)(f_{D_{j'},\lambda_{j}^*}^\diamond-f_\rho)\|_K^2]
   +
   \sum_{j'=1}^m\frac{|D_{j'}|^2}{|D|^2} E[\|(L_K+\lambda_{k_j^*} I)(f_{D_{j'},\lambda_{j}^*}-f^\diamond_{D_{j'},\lambda_{j}^*})\|_K^2] \\
   &\leq&
     (16\|h\|_\rho^2+19(\kappa M+\gamma)^2\tilde{C}_1)\lambda_{k^*_j}^{2r}. 
\end{eqnarray*}
Noting further $\lambda_{k^*_j}\leq\lambda_0$, we have
$$
  \max\{E[\|(\overline{f}_j-f_\rho)\|^2_\rho],E[\|(\overline{f}_j^\diamond-f_\rho)\|^2_\rho]\}
  \leq
 \tilde{C}_2\lambda_{k^*_j}^{2r} 
  \leq 
  \tilde{C}_2|D|^{-\frac{2r}{2r+s}},
$$
while 
\begin{eqnarray*}
    \max\{E[\|(\overline{f}_j-f_\rho)\|^2_K],E[\|(\overline{f}_j^\diamond-f_\rho)\|^2_K]\} 
  \leq 
  \tilde{C}_2\lambda_{k^*_j}^{2r-1} 
 \leq  
  \tilde{C}_2|D|^{-\frac{2r}{2r+s}},   
\end{eqnarray*}
where $\tilde{C}_2:=(16\|h\|_\rho^2+19(\kappa M+\gamma)^2\tilde{C}_1)$.
 This completes the proof of Proposition \ref{Prop:small-lambda}.    $\Box$
\end{proof}

\subsection{Generalization error analysis on over-estimated agents}
This parts devotes to generalization error analysis concerning $\overline{f}_j$ for $j$ satisfying $\lambda_{k_j^*}>\lambda_0$. The analysis is divided into three steps: approximation performance of the local approximation,   bounds concerning $g_{D_j,\lambda_j}$ and generalization error analysis. Similar as above, each step is described in a proposition, and the first one is as follows.

\begin{proposition}\label{Prop:Error-local-KRR}
If Assumption \ref{Assumption:regularity} holds with $\frac12\leq r\leq 1$, then
\begin{eqnarray*} 
    &&\|(L_K+\mu_j  I)^{1/2}( g^{loc}_{D_j,\lambda_k,\mu_j }-g_{D_j,\lambda_k})\|_K
    \leq
    b\|h_\rho\|_\rho\lambda_{k-1}\mu_j ^{r}\left( \mathcal Q_{D_j,\mu_j }^{2r}+  \mathcal Q^*_{D_j,\mu_j }\mathcal Q_{\Xi_L,\mu_j }(1+\mathcal Q_{D_j,\mu_j }\mathcal Q^*_{D_j,\mu_j })\mathcal Q_{D_j,\lambda_k}^{2r-1}    \right) \nonumber\\
    &+&
      b\mu_j ^{1/2}\lambda_k^{1/2}\mathcal P_{D_j,\lambda_k}\left(\mathcal Q_{D_j,\mu_j } \mathcal Q_{D,\lambda_k}+\mathcal Q^*_{D_j,\mu_j }\mathcal Q_{\Xi_L,\mu_j }(1+\mathcal Q_{D_j,\mu_j }\mathcal Q^*_{D_j,\mu_j })\mathcal Q_{D_j,\lambda_k}\right),
\end{eqnarray*}
and
\begin{eqnarray*} 
    \|(L_K+\mu_j  I)^{1/2}( g^{loc,\diamond}_{D_j,\lambda_k,\mu_j }-g_{D_j,\lambda_k}^\diamond)\|_K
    \leq 
  b\|h_\rho\|_\rho\lambda_{k-1}\mu_j ^{r}\left( \mathcal Q_{D_j,\mu_j }^{2r}+  \mathcal Q^*_{D_j,\mu_j }\mathcal Q_{\Xi_L,\mu_j }(1+\mathcal Q_{D_j,\mu_j }\mathcal Q^*_{D_j,\mu_j })\mathcal Q_{D_j,\lambda_k}^{2r-1}    \right).
\end{eqnarray*}
\end{proposition}

\begin{proof}
 The triangle inequality together with \eqref{Nystrom operator} yields  
\begin{eqnarray}\label{Error-decomposition-local-app}
   \| (L_K+\mu_j  I)^{1/2}((g^{loc}_{D_j,\lambda_k,\mu_j }-g_{D_j,\lambda_k})\|_K
   \leq
  \mathcal A_L(D_j,\lambda_k,\mu_j )+\mathcal C_L(D_j,\lambda_k,\mu_j ),
\end{eqnarray}
where  
\begin{eqnarray*}
   \mathcal A_L(D_j,\lambda_k,\mu_j )&=&\| (L_K+\mu_j  I)^{1/2} (h_{\Xi_L,\mu_j ,j}(L_{K,D_j})L_{K,D_j}-I)P_{\Xi_L,j}g_{D_j,\lambda_k}\|_K,\\
  \mathcal C_{L}(D_j,\lambda_k,\mu_j )&=&\|  (L_K+\mu_j  I)^{1/2}(h_{\Xi_{L},\mu_j ,j}(L_{K,D_j})L_{K,D_j}-I)(I-P_{\Xi_L,j})g_{D_j,\lambda_k}\|_K.
\end{eqnarray*} 
It
  follows from    \eqref{spectral-important-1} with $B=I$ that
$$ 
    P_{\Xi_L,j}=h_{\Xi_L,\mu_j ,j}(L_{K,D_j})(L_{K,D_j}+\mu_j    I)P_{\Xi_L,j}=h_{\Xi_L,\mu_j ,j}(L_{K,D_j})L_{K,D_j}P_{\Xi_L,j}
    +\mu_j  h_{\Xi_L,\mu_j ,j}(L_{K,D_j})P_{\Xi_L,j},
$$ 
which together with  \eqref{spetral definetion},
$P_{\Xi_L,j}=V_jV_j^T$  and $V_j^TV_j=I$ implies
\begin{equation*} 
  (h_{\Xi_L,\mu_j ,j}(L_{K,D_j})L_{K,D_j}-I)P_{\Xi_L,j}
  =\mu_j  h_{\Xi_L,\mu_j ,j}(L_{K,D_j})P_{\Xi_L,j}=\mu_j  h_{\Xi_L,\mu_j ,j}(L_{K,D_j}).
\end{equation*}
Therefore,  $\|f\|_\rho\leq\|(L_K+\mu_j  I)^{1/2}f_\rho\|_K$, \eqref{Def.QD} and \eqref{spectral-important-2} yield
\begin{eqnarray}\label{Bound.A}
      \mathcal A_L(D_j,\lambda_k,\mu_j )
      &\leq&
      \mu_j  \|(L_K+\mu_j  I)^{1/2}h_{\Xi_L,\mu_j ,j}(L_{K,D_j})g_{D_j,\lambda_k}\|_K
      \leq 
      \mu_j  \mathcal Q_{D_j,\mu_j }\|(L_{K,D_j}+\mu_j  I)^{-1/2}g_{D_j,\lambda_k}\|_K \nonumber\\
      &\leq&
      b\mathcal Q_{D_j,\mu_j }^{2r}\mu_j ^r\lambda_{k-1} \|h_\rho\|_\rho
    +
    b\mu_j ^{1/2}\lambda_k^{1/2}\mathcal Q_{D_j,\mu_j } \mathcal Q_{D,\lambda_k}\mathcal P_{D_j,\lambda_k}.
\end{eqnarray}
Due to Lemma \ref{Lemma:Projection general},  there holds
$$ 
   \|(I-P_{\Xi_L,j})(L_K+\mu_j  I)^{1/2}\|\leq 
   \|(L_{K,\Xi_L}+\mu_j  I)^{-1/2}(L_{K}+\mu_j 
        I)^{1/2}\|
        \leq \mathcal Q_{\Xi_L,\mu_j }\mu_j ^{1/2}  
$$ 
and
$$ 
        \|(I-P_{\Xi_L,j})(L_{K,D_j}+\mu_j 
        I)^{1/2}\|
       \leq
        \mu_j ^{1/2}\mathcal Q^*_{D_j,\mu_j }\|(L_{K,\Xi_L}+\mu_j  I)^{-1/2}(L_{K}+\mu_j 
        I)^{1/2}\|\leq \mu_j ^{1/2}\mathcal Q^*_{D_j,\mu_j }\mathcal Q_{\Xi_L,\mu_j }.
$$
The above estimates together with \eqref{Def.QD}, \eqref{spectral-important-2}, \eqref{bound.error-difference} and \eqref{Projection-property} with $u=2$ yields
\begin{eqnarray*}
    &&\mathcal C_L(D_j,\lambda_k,\mu_j )
     \leq 
    \|(L_K+\mu_j  I)^{1/2} h_{\Xi_{L},\mu_j ,j}(L_{K,D_j})L_{K,D_j}   (I-P_{\Xi_L,j})g_{D_j,\lambda_k}\|_K
    +
    \|(L_K+\mu_j  I)^{1/2}(I-P_{\Xi_L,j})g_{D_j,\lambda_k}\|_K\\
    &\leq&
   \mu_j ^{1/2} \mathcal Q_{\Xi_L,\mu_j }(1+\mathcal Q_{D_j,\mu_j }\mathcal Q^*_{D_j,\mu_j })  \|(I-P_{\Xi_L,j})g_{D_j,\lambda_k}\|_K\\
   &\leq&
   \mu_j  \mathcal Q^*_{D_j,\mu_j }\mathcal Q_{\Xi_L,\mu_j }(1+\mathcal Q_{D_j,\mu_j }\mathcal Q^*_{D_j,\mu_j }) \|(L_{K,D_j}+\mu_j  I)^{-1/2}g_{D_j,\lambda_k}\|_K
   \\
   &\leq&
   b  \mathcal Q^*_{D_j,\mu_j }\mathcal Q_{\Xi_L,\mu_j }(1+\mathcal Q_{D_j,\mu_j }\mathcal Q^*_{D_j,\mu_j })
   ( \mu_j ^{1/2}\lambda_k^{1/2}\mathcal Q_{D_j,\lambda_k}\mathcal P_{D_j,\lambda_k}
  +
   \mathcal Q_{D_j,\mu_j }^{2r-1}\lambda_{k-1}\mu_j ^{r}\|h_\rho\|_\rho).
\end{eqnarray*}
Inserting the above estimate and \eqref{Bound.A} into \eqref{Error-decomposition-local-app}, we get
\begin{eqnarray*} 
    &&\|(L_K+\mu_j  I)^{1/2}(g^{loc}_{D_j,\lambda_k,\mu_j }-g_{D_j,\lambda_k})\|_\rho
    \leq
    b\|h_\rho\|_\rho\lambda_{k-1}\mu_j ^{r}\left( \mathcal Q_{D_j,\mu_j }^{2r}+  \mathcal Q^*_{D_j,\mu_j }\mathcal Q_{\Xi_L,\mu_j }(1+\mathcal Q_{D_j,\mu_j }\mathcal Q^*_{D_j,\mu_j })\mathcal Q_{D_j,\lambda_k}^{2r-1}    \right)  \\
    &+&
      b\mu_j ^{1/2}\lambda_k^{1/2}\mathcal P_{D_j,\lambda_k}\left(\mathcal Q_{D_j,\mu_j } \mathcal Q_{D,\lambda_k}+\mathcal Q^*_{D_j,\mu_j }\mathcal Q_{\Xi_L,\mu_j }(1+\mathcal Q_{D_j,\mu_j }\mathcal Q^*_{D_j,\mu_j })\mathcal Q_{D_j,\lambda_k}\right).
\end{eqnarray*}

The same method as above can deduce 
\begin{eqnarray*} 
    \|(L_K+\mu_j  I)^{1/2}( g^{loc,\diamond}_{D_j,\lambda_k,\mu_j }-g_{D_j,\lambda_k}^\diamond)\|_K
    \leq 
  b\|h_\rho\|_\rho\lambda_{k-1}\mu_j ^{r}\left( \mathcal Q_{D_j,\mu_j }^{2r}+  \mathcal Q^*_{D_j,\mu_j }\mathcal Q_{\Xi_L,\mu_j }(1+\mathcal Q_{D_j,\mu_j }\mathcal Q^*_{D_j,\mu_j })\mathcal Q_{D_j,\lambda_k}^{2r-1}    \right) 
\end{eqnarray*}
directly.
This finishes the proof of Proposition \ref{Prop:Error-local-KRR}.
$\Box$
\end{proof}
To present the next proposition, we need the following lemma on bounds of operator difference, whose proof  is given in Appendix.   
\begin{lemma}\label{Lemma:Q333-111}
Under  Assumption \ref{Assumption:boundedness}, Assumption \ref{Assumption:regularity} with $\frac12\leq r\leq 1$ and Assumption \ref{Assumption:effective dimension} with $0\leq s\leq 1$, if  $\Xi_L=\{\xi_\ell\}_{\ell=1}^L$ is a set of i.i.d. random variables with $
     L \geq \max_{j=1,\dots,m} |D_j|,
$
\eqref{res.theorem} holds and $\mu_j $ satisfies \eqref{Selection-mu}, then 
\begin{eqnarray*}
     E\left[\left( \sum_{k=k_j^*+1}^{k_0}\lambda_{k-1}\left( \mathcal Q_{D_j,\mu_j }^{2r}+  \mathcal Q^*_{D_j,\mu_j }\mathcal Q_{\Xi_L,\mu_j }(1+\mathcal Q_{D_j,\mu_j }\mathcal Q^*_{D_j,\mu_j })\mathcal Q_{D_j,\lambda_k}^{2r-1}    \right)\right)^2 \right]
      \leq 
     2b^{-2}(2+6\sqrt{2})^2 \log^2 |D|, 
\end{eqnarray*}
\begin{eqnarray*}
   &&E\left[\left(\sum_{k=k_j^*+1}^{k_0} \lambda_k^{1/2}\mathcal P_{D_j,\lambda_k}\left(\mathcal Q_{D_j,\mu_j } \mathcal Q_{D,\lambda_k}+\mathcal Q^*_{D_j,\mu_j }\mathcal Q_{\Xi_L,\mu_j }(1+\mathcal Q_{D_j,\mu_j }\mathcal Q^*_{D_j,\mu_j })\mathcal Q_{D_j,\lambda_k}\right)\right)^2\right] \\
   &\leq&
   48(2+6\sqrt{2})^2(1+C_0b)^2 (\kappa M +\gamma)^2   \left(|D|^{\frac1{4r+2s}}|D_j|^{-1}+|D|^\frac{s+1}{4r+2s}|D_j|^{-\frac12}\right)^2,
\end{eqnarray*}
$$
E\left[\sum_{j=1}^m\frac{|D_j|}{|D|}   \left(\sum_{k=k_j^*+1}^{k_0}\mathcal Q_{D_j,\lambda_k}^*\lambda_{k-1}\overline{\mathcal W}_{D,\lambda_k}^{1/2}\right)^2\right]
\leq
  16(40\sqrt{2}b^{-1})^2\Gamma(5) |D|^{-\frac{2r}{2r+s}}, 
$$
and
$$
E\left[\sum_{j=1}^m\frac{|D_j|}{|D|}   \left(\sum_{k=k_j^*+1}^{k_0}\mathcal Q_{D_j,\lambda_k}^*\lambda_{k-1}\lambda_k^{1/2}\overline{\mathcal W}_{D,\lambda_k}^{1/2}\right)^2\right]
\leq
  16(40\sqrt{2/b})^2\Gamma(5) |D|^{-\frac{2r-1}{2r+s}}. 
$$
\end{lemma}

\begin{proof} 
Due to Lemma \ref{Lemma:Q-1111}, with confidence $1-\delta$ for $\delta$ satisfying \eqref{res.11334}, there holds
\begin{eqnarray*}
    \mathcal Q_{D_j,\mu_j }^{2r}+  \mathcal Q^*_{D_j,\mu_j }\mathcal Q_{\Xi_L,\mu_j }(1+\mathcal Q_{D_j,\mu_j }\mathcal Q^*_{D_j,\mu_j })\mathcal Q_{D_j,\lambda_k}^{2r-1}  \leq
    (2+6\sqrt{2}),
\end{eqnarray*}
which means
\begin{eqnarray*}
      && \sum_{k=k_j^*+1}^{k_0}\lambda_{k-1}\left( \mathcal Q_{D_j,\mu_j }^{2r}+  \mathcal Q^*_{D_j,\mu_j }\mathcal Q_{\Xi_L,\mu_j }(1+\mathcal Q_{D_j,\mu_j }\mathcal Q^*_{D_j,\mu_j })\mathcal Q_{D_j,\lambda_k}^{2r-1}    \right)
   \leq 
   b^{-1}(2+6\sqrt{2})\sum_{k=k_j^*+1}^{k_0}\frac{1}{k-1}\\
   &\leq&
    b^{-1}(2+6\sqrt{2}) \log k_0
    \leq
    b^{-1}(2+6\sqrt{2})\frac{2r}{2r+s} \log |D|.
\end{eqnarray*}
This together with Lemma \ref{Lemma:prob-to-exp} and the definitions of $K^*$ and $k_0$ shows
\begin{eqnarray*}
  && E\left[ \left(\sum_{k=k_j^*+1}^{k_0}\lambda_{k-1}\left( \mathcal Q_{D_j,\mu_j }^{2r}+  \mathcal Q^*_{D_j,\mu_j }\mathcal Q_{\Xi_L,\mu_j }(1+\mathcal Q_{D_j,\mu_j }\mathcal Q^*_{D_j,\mu_j })\mathcal Q_{D_j,\lambda_k}^{2r-1} \right)   \right)^2 \right] 
  \leq
   2b^{-2}(2+6\sqrt{2})^2 \log^2 |D|.
\end{eqnarray*}
Similarly, we get from   Lemma \ref{Lemma:Q-1111} and Assumption \ref{Assumption:effective dimension} that with confidence $1-\delta$ for $\delta$ satisfying \eqref{res.11334}, there holds
\begin{eqnarray*}
      &&\mathcal P_{D_j,\lambda_k}\left(\mathcal Q_{D_j,\mu_j } \mathcal Q_{D,\lambda_k}+\mathcal Q^*_{D_j,\mu_j }\mathcal Q_{\Xi_L,\mu_j }(1+\mathcal Q_{D_j,\mu_j }\mathcal Q^*_{D_j,\mu_j })\mathcal Q_{D_j,\lambda_k}\right)\\
      &\leq&
      2(2+6\sqrt{2}) (\kappa M +\gamma) \left(\frac{1}{|D_j|\sqrt{\lambda_k}}+\frac{\sqrt{\mathcal
        N(\lambda_k)}}{\sqrt{|D_j|}}\right) \log
               \bigl(2/\delta\bigr)\\
            &\leq&
       2(2+6\sqrt{2}) (\kappa M +\gamma) \left(\frac{1}{|D_j|\sqrt{\lambda_k}}+\frac{C_0\lambda_k^{-\frac{s}{2}}}{\sqrt{|D_j|}}\right)\log
               \bigl(2/\delta\bigr),    
\end{eqnarray*}
 which implies
\begin{eqnarray*}
    &&\sum_{k=k_j^*+1}^{k_0} \lambda_k^{1/2}\mathcal P_{D_j,\lambda_k}\left(\mathcal Q_{D_j,\mu_j } \mathcal Q_{D,\lambda_k}+\mathcal Q^*_{D_j,\mu_j }\mathcal Q_{\Xi_L,\mu_j }(1+\mathcal Q_{D_j,\mu_j }\mathcal Q^*_{D_j,\mu_j })\mathcal Q_{D_j,\lambda_k}\right)\log
               \bigl(2/\delta\bigr)\\
    &\leq&
   2(2+6\sqrt{2})(1+C_0b) (\kappa M +\gamma)   \left(|D|^{\frac1{4r+2s}}|D_j|^{-1}+|D|^\frac{s+1}{4r+2s}|D_j|^{-\frac12}\right)\log
               \bigl(2/\delta\bigr).
\end{eqnarray*} 
Hence, similar proof skills as those in Lemma \ref{Lemma:Q222-111} together with Lemma \ref{Lemma:prob-to-exp} yield
\begin{eqnarray*}
   &&E\left[\left(\sum_{k=k_j^*+1}^{k_0} \lambda_k^{1/2}\mathcal P_{D_j,\lambda_k}\left(\mathcal Q_{D_j,\mu_j } \mathcal Q_{D,\lambda_k}+\mathcal Q^*_{D_j,\mu_j }\mathcal Q_{\Xi_L,\mu_j }(1+\mathcal Q_{D_j,\mu_j }\mathcal Q^*_{D_j,\mu_j })\mathcal Q_{D_j,\lambda_k}\right)\right)^2\right] \\
   &\leq&
   48(2+6\sqrt{2})^2(1+C_0b)^2 (\kappa M +\gamma)^2   \left(|D|^{\frac1{4r+2s}}|D_j|^{-1}+|D|^\frac{s+1}{4r+2s}|D_j|^{-\frac12}\right)^2.
\end{eqnarray*}
Due to \eqref{Def.WD}, Lemma \ref{Lemma:Q-1111} and   Assumption \ref{Assumption:effective dimension}, 
 with confidence $1-\delta$, there holds
\begin{eqnarray*} 
	\mathcal W_{D_{j'},\lambda_k}&\leq&
 \frac{1}{|D_{j'}|\sqrt{ \lambda_k}}+\left( 1+ \frac{1}{\sqrt{\lambda_k|D_{j'}|}}\right) \sqrt{\frac{(1+4\sqrt{\eta_{\delta/4}}\vee\eta_{\delta/4}^2)\sqrt{\max\{\mathcal N(\lambda_j),1\}}}{|D_{j'}|}}\nonumber\\
 &\leq& 
 2\left(\frac{1}{\sqrt{\lambda_k}|D_{j'}|}+\frac{(1+4(1+1/(\lambda_k|D_{j}|)))\sqrt{C_0}\lambda_k^{-s/2}(1+8\sqrt{1/\lambda_k|D_{j'}|})}{\sqrt{|D_{j'}|}}\right)\log^2\frac{16}\delta.
\end{eqnarray*}
Moreover, $\lambda_k>\lambda_0$ and \eqref{res.theorem} yield $\lambda_k|D_{j'}|\geq   8$.
Therefore, we have from \eqref{res.theorem}, the definition of $\overline{\mathcal W}_{D,\lambda_k}$ in \eqref{global-111} and Lemma \ref{Lemma:Q-1111} again that 
\begin{eqnarray*}
	&&\sum_{k=k^*+1}^{k_0}\mathcal Q_{D_j,\lambda_k}^* \lambda_{k-1}\overline{\mathcal W}_{D,\lambda_k}^{1/2}
  \leq 20\sqrt{2}b^{-1} \sum_{k=k^*+1}^{k_0} \left(\sum_{{j'=1}}^m\frac{|D_{j'}|^2}{|D|^2}  \left(\frac{k^{-3/2}}{|D_{j'}|^2}+\frac{k^{s-2}}{|D_{j'}|} \right)\right)^{1/2} \log^2\frac{16}\delta\\
 &\leq&
20\sqrt{2}b^{-1} \sum_{k=k^*+1}^{k_0} 
 \left(\frac{k^{-3/4}\sqrt{m}}{|D|}+\frac{k^{s/2-1}}{|D|^{1/2}}\right)\log^2\frac{16}\delta
 \leq
 20\sqrt{2}b^{-1}\left(m^{1/2}|D|^{-\frac{2r+s-1}{2r+s}}+
 |D|^{-\frac{r}{2r+s}}\right)\\
 &\leq&
 40\sqrt{2}b^{-1}|D|^{-\frac{r}{2r+s}} 
\end{eqnarray*}
holds for any $j=1,\dots,m$.
Therefore, we obtain from Lemma \ref{Lemma:prob-to-exp} that
$$
E\left[\sum_{j=1}^m\frac{|D_j|}{|D|}   \left(\sum_{k=k_j^*+1}^{k_0}\mathcal Q_{D_j,\lambda_k}^*\lambda_{k-1}\overline{\mathcal W}_{D,\lambda_k}^{1/2}\right)^2\right]
\leq
  16(40\sqrt{2}b^{-1})^2\Gamma(5) |D|^{-\frac{2r}{2r+s}}. 
$$
Similarly, we can derive
$$
E\left[\sum_{j=1}^m\frac{|D_j|}{|D|}   \left(\sum_{k=k_j^*+1}^{k_0}\mathcal Q_{D_j,\lambda_k}^*\lambda_{k-1}\lambda_k^{1/2}\overline{\mathcal W}_{D,\lambda_k}^{1/2}\right)^2\right]
\leq
  16(40\sqrt{2}b^{-1/2})^2\Gamma(5) |D|^{-\frac{2r-1}{2r+s}}. 
$$
This completes the proof of Lemma \ref{Lemma:Q333-111}. $\Box$
\end{proof}

Based on Proposition \ref{Prop:Error-local-KRR} and Lemma \ref{Lemma:Q333-111}, we are in a position to present the following proposition.

\begin{proposition}\label{Prop:bound-for-difference}
Under Assumption \ref{Assumption:boundedness}, Assumption  \ref{Assumption:regularity} with $\frac12\leq r\leq 1$ and Assumption \ref{Assumption:effective dimension}  with $0<s\leq 1$, if  $\Xi_L=\{\xi_\ell\}_{\ell=1}^L$ is a set of i.i.d. random variables with $
     L \geq \max_{j=1,\dots,m} |D_j|,
$
\eqref{res.theorem} holds,   $\mu_j $ satisfies \eqref{Selection-mu}, and $\lambda_{k_j^*}>\lambda_0$, then  
\begin{eqnarray*}
 E\left[\left\|\sum_{j=1}^m\frac{|D_j|}{|D|}\sum_{k=k_j^*+1}^{k_0}
  g_{D_j,\lambda_k}
   \right\|_K^2\right]&\leq&
   \tilde{C}_3 |D|^{-\frac{2r}{2r+s}},\\
   E\left[\left\|
    \sum_{j=1}^m\frac{|D_j|}{|D|}\sum_{k=k_j^*+1}^{k_0}g_{D_j,\lambda_k}
   \right\|_\rho^2\right]
   &\leq&
   \tilde{C}_3|D|^{-\frac{2r-1}{2r+s}},
\end{eqnarray*}
where $ \tilde{C}_3$ is a constant depending only on  $\|h_\rho\|_\rho,b,\kappa,C_0,M,\gamma$.
\end{proposition}
\begin{proof}
Since $k_j^*$ is the first (or largest) $k$ 
  satisfying \eqref{stopping-1}, for any $k>k_j^*$, there holds
$$
   \|(L_{K,D_j}+\lambda_k I)^{1/2}g^{global}_{D,\lambda_k,\mu}\|_K^2 
    <
    C_{LP}  \lambda_{k-1}^2\overline{\mathcal W}_{D,\lambda_k}.
$$ 
Moreover, for any $k\in\mathbb N$, we have
$$
   \sum_{j=1}^m\frac{|D_j|}{|D|}g_{D_j,\lambda_k}=\sum_{j=1}^m\frac{|D_j|}{|D|}(g_{D_j,\lambda_k}-g_{D_j,\lambda_k,\mu_j }^{loc})+\sum_{j=1}^m\frac{|D_j|}{|D|}g_{D,\lambda_k,\mu}^{global}.
$$
Then, we get from \eqref{Def:Q*}, Lemma \ref{Lemma:power-synthesization} with $g_j:=(L_{K}+\mu_j  I)^{1/2}\sum_{k=k_j^*+1}^{k_0}(g_{D_j,\lambda_k}-g_{D_j,\lambda_k,\mu_j }^{loc})$ and Jensen' inequality 
  that 
\begin{eqnarray}\label{large.3.1-rho-error-dec}
    &&E\left[\left\|
    \sum_{j=1}^m\frac{|D_j|}{|D|}\sum_{k=k_j^*+1}^{k_0}g_{D_j,\lambda_k}
   \right\|_\rho^2\right]
   \leq
   2E\left[
\left\|\sum_{k=k_j^*+1}^{k_0}\sum_{j=1}^{m}\frac{|D_j|}{|D|}(L_{K}+\mu_j  I)^{1/2}(g_{D_j,\lambda_k}-g_{D_j,\lambda_k,\mu_j }^{loc})\right\|^2_K\right] \nonumber\\
   &+&
   2E\left[\sum_{j=1}^{m} \frac{|D_j|}{|D|} \left(\sum_{k=k_j^*+1}^{k_0}\mathcal Q_{D_j,\lambda_k}^*\|(L_{K,D_j}+\lambda_k I)^{1/2} g^{global}_{D,\lambda_k,\mu}\|_{K}\right)^2
    \right]\nonumber\\
    &\leq&
    2E\left[
\sum_{j=1}^{m}\frac{|D_j|^2}{|D|^2}\left\|\sum_{k=k_j^*+1}^{k_0}(L_{K}+\mu_j  I)^{1/2}(g_{D_j,\lambda_k}-g_{D_j,\lambda_k,\mu_j }^{loc})\right\|^2_K\right] \nonumber\\
&+&
   4E\left[
  \sum_{j=1}^{m}\frac{|D_j|}{|D|}\left\|\sum_{k=k_j^*+1}^{k_0}(L_{K}+\mu_j  I)^{1/2}(g_{D_j,\lambda_k}^\diamond-g_{D_j,\lambda_k,\mu_j }^{loc,\diamond})\right\|^2_K\right]\nonumber\\
  &+&2C_{LP}E\left[\sum_{j=1}^m\frac{|D_j|}{|D|}   \left(\sum_{k=k_j^*+1}^{k_0}\mathcal Q_{D_j,\lambda_k}^*\lambda_{k-1}\overline{\mathcal W}_{D,\lambda_k}^{1/2}\right)^2\right].  
\end{eqnarray}
Similarly, we have
\begin{eqnarray}\label{large.3.1-K-error-dec}
    &&E\left[\left\|
    \sum_{j=1}^m\frac{|D_j|}{|D|}\sum_{k=k_j^*+1}^{k_0}g_{D_j,\lambda_k}
   \right\|_K^2\right]
   \leq
    2\mu_j ^{-1}E\left[
\sum_{j=1}^{m}\frac{|D_j|^2}{|D|^2}\left\|\sum_{k=k_j^*+1}^{k_0}(L_{K}+\mu_j  I)^{1/2}(g_{D_j,\lambda_k}-g_{D_j,\lambda_k,\mu_j }^{loc})\right\|^2_K\right]\nonumber\\
&+&
   4\mu_j ^{-1}E\left[
  \sum_{j=1}^{m}\frac{|D_j|}{|D|}\left\|\sum_{k=k_j^*+1}^{k_0}(L_{K}+\mu_j  I)^{1/2}(g_{D_j,\lambda_k}^\diamond-g_{D_j,\lambda_k,\mu_j }^{loc,\diamond})\right\|^2_K\right]\nonumber\\
  &+&
  2C_{LP}E\left[ \sum_{j=1}^m\frac{|D_j|}{|D|}  \left(\sum_{k=k_j^*+1}^{k_0}\mathcal Q_{D_j,\lambda_k}^*\lambda_k^{-1/2}\lambda_{k-1}\overline{\mathcal W}_{D,\lambda_k}^{1/2}\right)^2\right].
\end{eqnarray}
However, Proposition \ref{Prop:Error-local-KRR} together with Lemma \ref{Lemma:Q333-111} yields that 
\begin{eqnarray}\label{large.3.3}
    &&E\left[
  \sum_{j=1}^{m}\frac{|D_j|}{|D|}\left\|\sum_{k=k_j^*+1}^{k_0}(L_{K}+\mu_j  I)^{1/2}(g_{D_j,\lambda_k}^\diamond-g_{D_j,\lambda_k,\mu_j }^{loc,\diamond})\right\|^2_K\right]\nonumber\\
  &\leq&
  b^2\|h_\rho\|_\rho^2\mu_j ^{2r}\sum_{j=1}^{m}\frac{|D_j|}{|D|} E\left[ \left(\sum_{k=k_j^*+1}^{k_0}\lambda_{k-1}\left( \mathcal Q_{D_j,\mu_j }^{2r}+  \mathcal Q^*_{D_j,\mu_j }\mathcal Q_{\Xi_L,\mu_j }(1+\mathcal Q_{D_j,\mu_j }\mathcal Q^*_{D_j,\mu_j })\mathcal Q_{D_j,\lambda_k}^{2r-1}    \right)\right)^2 \right]\nonumber\\
  &\leq&
  2(2+6\sqrt{2})^2  \|h_\rho\|_\rho^2\mu_j ^{2r}  \log^2 |D|
\end{eqnarray}
and
\begin{eqnarray}\label{large.3.4}
    &&E\left[
  \sum_{j=1}^{m}\frac{|D_j|}{|D|}\left\|\sum_{k=k_j^*+1}^{k_0}(L_{K}+\mu_j  I)^{1/2}(g_{D_j,\lambda_k}^\diamond-g_{D_j,\lambda_k,\mu_j }^{loc,\diamond})\right\|^2_K\right]\nonumber\\
  &\leq&
  \sum_{j=1}^{m}\frac{|D_j|}{|D|} 
   E\left[
  \left(
   \sum_{k=k_j^*+1}^{k_0}
   b\|h_\rho\|_\rho\lambda_{k-1}\mu_j ^{r}\left( \mathcal Q_{D_j,\mu_j }^{2r}+  \mathcal Q^*_{D_j,\mu_j }\mathcal Q_{\Xi_L,\mu_j }(1+\mathcal Q_{D_j,\mu_j }\mathcal Q^*_{D_j,\mu_j })\mathcal Q_{D_j,\lambda_k}^{2r-1}    \right) \right.\right.\nonumber\\
    &+&
     \left.\left. \sum_{k=k_j^*+1}^{k_0}b\mu_j ^{1/2}\lambda_k^{1/2}\mathcal P_{D_j,\lambda_k}\left(\mathcal Q_{D_j,\mu_j } \mathcal Q_{D,\lambda_k}+\mathcal Q^*_{D_j,\mu_j }\mathcal Q_{\Xi_L,\mu_j }(1+\mathcal Q_{D_j,\mu_j }\mathcal Q^*_{D_j,\mu_j })\mathcal Q_{D_j,\lambda_k}\right)
\right)^2\right]\nonumber\\
  &\leq&
  2(2+6\sqrt{2})^2  \|h_\rho\|_\rho^2\mu_j ^{2r}  \log^2 |D|\nonumber\\
  &+&96b^2 (2+6\sqrt{2})^2(1+C_0b)^2 (\kappa M +\gamma)^2 \mu_j   \sum_{j=1}^{m}\frac{|D_j|}{|D|} \left(|D|^{\frac1{4r+2s}}|D_j|^{-1}+|D|^\frac{s+1}{4r+2s}|D_j|^{-\frac12}\right)^2.
\end{eqnarray}
Plugging \eqref{large.3.3} and \eqref{large.3.4} into \eqref{large.3.1-rho-error-dec} and noting Lemma \ref{Lemma:Q333-111}, \eqref{res.theorem} and \eqref{Selection-mu}, we obtain
\begin{eqnarray*}
   E\left[\left\|
    \sum_{j=1}^m\frac{|D_j|}{|D|}\sum_{k=k_j^*+1}^{k_0}g_{D_j,\lambda_k}
   \right\|_\rho^2\right]
   \leq
   \tilde{C}_3  |D|^{-\frac{2r}{2r+s}},
\end{eqnarray*}
where 
$
   \tilde{C}_3:=8(2+6\sqrt{2})^2  \|h_\rho\|_\rho^2
   +768 b^2 (2+6\sqrt{2})^2(1+C_0b)^2 (\kappa M +\gamma)^2+16(40\sqrt{2}b^{-1})^2\Gamma(5). 
$
Moreover, inserting \eqref{large.3.3} and \eqref{large.3.4} into \eqref{large.3.1-K-error-dec}, we obtain from Lemma \ref{Lemma:Q333-111} that
\begin{eqnarray*}
   E\left[\left\|
    \sum_{j=1}^m\frac{|D_j|}{|D|}\sum_{k=k_j^*+1}^{k_0}g_{D_j,\lambda_k}
   \right\|_K^2\right]
   \leq
   \tilde{C}_3|D|^{-\frac{2r-1}{2r+s}}.
\end{eqnarray*}
This completes the proof of Proposition \ref{Prop:bound-for-difference}.
$\Box$
\end{proof}

To derive the generalization error for $\overline{f}_j$, we need the following standard results in distributed learning, whose proof can be found in \cite{wang2023adaptive}.
\begin{lemma}\label{Lemma:distributed-rate-fixed}
Under Assumption  \ref{Assumption:boundedness}, Assumption \ref{Assumption:regularity} with $\frac12\leq r\leq 1$ and Assumption \ref{Assumption:effective dimension} with $0<s\leq 1$, if \eqref{res.theorem} holds and $\lambda_0=\lambda_{k_0}$ with $k_0=\left[|D|^{\frac{1}{2r+s}} \right]$, then
$$
    E\left[  \left\|(L_K+\lambda_0 I)^{1/2} \left(\sum_{j=1}^{m}\frac{|D_j|}{|D|}f_{D_j,\lambda_0}-f_\rho\right)\right\|_K^2\right] 
       \leq \tilde{C}_4 |D|^{-\frac{2r}{2r+s}},
$$
where $\tilde{C}_4$ is a constant depending only on $r,s,\kappa,\gamma,M,b$.
\end{lemma}

We then derive the generalization error for $\overline{f}_j$ in the following proposition. 
\begin{proposition}\label{Prop:large-lambda}
Let $1\leq j\leq m$ and $\lambda_0=\lambda_{k_0}$ for $k_0=[|D|^{1/(2r+s)}]$.  
Under Assumption \ref{Assumption:boundedness}, Assumption \ref{Assumption:regularity}   with $\frac12\leq r\leq1$ and Assumption \ref{Assumption:effective dimension} with $0\leq s\leq 1$, if $\lambda_{k_j^*}>\lambda_0$, \eqref{res.theorem} holds,  $\mu_j $ satisfies \eqref{Selection-mu} and $\Xi_L$ is a set of i.i.d. random variables with $L\geq \max\{|D_1|,\dots,|D_m|\}$, then
\begin{equation}\label{large-rho-norm}
    E[\|(L_K+\lambda_0)^{1/2}(\overline{f}_j-f_\rho)\|^2_\rho] 
  \leq 
  \tilde{C}_5|D|^{-\frac{2r}{2r+s}}, 
\end{equation}
 where $\tilde{C}_5$ is a constant depending only on $M,\kappa,b,\|h\|_\rho,\gamma,r,s$.
\end{proposition}

\begin{proof}
 The triangle inequality follows
\begin{eqnarray*} 
        &&\|(L_K+\lambda_0 I)^{1/2}(\overline{f}_j-f_\rho)\|_K\\
      &\leq&
     \left\|(L_K+\lambda_0 I)^{1/2}\left(\sum_{j'=1}^m \frac{|D_{j'}|}{|D|}
      f_{D_{j'},\lambda_j^*}-\sum_{j'=1}^m \frac{|D_{j'}|}{|D|}
      f_{D_{j'},\lambda_0}\right)\right\|_K 
      +
     \left\|(L_K+\lambda_0 I)^{1/2}\left(\sum_{j'=1}^m \frac{|D_{j'}|}{|D|}
      f_{D_{j'},\lambda_0}-f_\rho\right)\right\|_K \\
       &\leq&
       \left\|(L_K+\lambda_0 I)^{1/2}\left(\sum_{j'=1}^m \frac{|D_{j'}|}{|D|}
      f_{D_{j'},\lambda_0}-f_\rho\right)\right\|_K
      +
      \left\|(L_K+\lambda_0 I)^{1/2} \sum_{k=k^*_j+1}^{k_0} \left(\sum_{j'=1}^m \frac{|D_{j'}|}{|D|}g_{D_{j'},\lambda_{k}}\right)\right\|_K.
\end{eqnarray*}
Combining the above inequality with Lemma \ref{Lemma:distributed-rate-fixed} and Proposition \ref{Prop:bound-for-difference}, we obtain
$$
    E\left[\|(L_K+\lambda_0 I)^{1/2}(\overline{f}_j-f_\rho)\|_K^2\right]
    \leq
    (4\tilde{C}_3+2\tilde{C}_4)|D|^{-\frac{2r}{2r+s}}.
$$
This completes the proof of Proposition \ref{Prop:large-lambda} with $\tilde{C}_5:=24\tilde{C}_3+2\tilde{C}_4$. $\Box$
\end{proof}

\subsection{Proof of Theorem \ref{Theorem:Optimal-Rate-adaptive}}
Theorem \ref{Theorem:Optimal-Rate-adaptive} is then a direct consequence of Proposition \ref{Prop:large-lambda} and Proposition \ref{Prop:small-lambda}.

\begin{proof}[Proof of Theorem \ref{Theorem:Optimal-Rate-adaptive}]
 It follows from Lemma \ref{Lemma:power-synthesization}, \eqref{KRR-local:operator} and Jensen's inequality   that 
\begin{eqnarray*} 
      && E[\|(L_K+\lambda_0 I)^{1/2}(\overline{f}_{D,\vec{\lambda}^*}-f_\rho)\|_K^2]\\
      &\leq &
     \sum_{j=1}^m\frac{|D_j|}{|D|}
      \|(L_K+\lambda_0 I)^{1/2}(E[\overline{f}_j]-f_\rho)\|_K^2]
     +
     \sum_{j=1}^m\frac{|D_j|^2}{|D|^2}E[\|(L_K+\lambda_0 I)^{1/2}(\overline{f}_j-f_\rho)\|_K^2] \nonumber \\
     &\leq&
 \sum_{j:\lambda_j^*\leq\lambda_0}  \frac{|D_j|}{|D|}
      \left(\|(L_K+\lambda_0 I)^{1/2}(E[\overline{f}_j]-f_\rho)\|_K^2]
     +   
     \frac{|D_j|}{|D|}E[\|(L_K+\lambda_0 I)^{1/2}(\overline{f}_j-f_\rho)\|^2_K]\right)\\
     &+&
     2\sum_{j:\lambda_j^*>\lambda_0}\frac{|D_j|}{|D|}
     E[\|(L_K+\lambda_0 I)^{1/2}(\overline{f}_j-f_\rho)\|_K^2],
\end{eqnarray*}
 But Proposition \ref{Prop:small-lambda} yields
$$
   \sum_{j:\lambda_j^*\leq\lambda_0}  \frac{|D_j|}{|D|}
      \left(\|E[(L_K+\lambda_0 I)^{1/2} (\overline{f}_j]-f_\rho)\|_K^2]
     +   
     \frac{|D_j|}{|D|}E[\|(L_K+\lambda_0 I)^{1/2}(\overline{f}_j-f_\rho)\|^2_K]\right)
     \leq 2\tilde{C}_2|D|^{-\frac{2r}{2r+s}},
$$
while Proposition \ref{Prop:large-lambda} implies
$$
  2\sum_{j:\lambda_j^*>\lambda_0}\frac{|D_j|}{|D|}
     E[\|(L_K+\lambda_0 I)^{1/2}(\overline{f}_j-f_\rho)\|_K^2]
     \leq
     2\tilde{C}_5|D|^{-\frac{2r}{2r+s}}.
$$
Hence, we get 
\begin{eqnarray*} 
      E[\|(L_K+\lambda_0 I)^{1/2}(\overline{f}_{D,\vec{\lambda}^*}-f_\rho)\|_K^2]
     \leq \bar{C}|D|^{-\frac{2r}{2r+s}}.
\end{eqnarray*} 
This together with $\|f\|_\rho=\|L_K^{1/2}f\|_K$ yields
$$
    E\left[\|(\overline{f}_j-f_\rho)\|_\rho^2\right]
    \leq
    (4\tilde{C}_3+2\tilde{C}_4)|D|^{-\frac{2r}{2r+s}} 
$$
and together with $\|f\|_K\leq \lambda_0^{-1/2}\|(L_K+\lambda_0)^{1/2}f\|_K$ with $\lambda_0\sim|D|^{-\frac{1}{2r+s}}$ follows
$$
    E\left[\|(\overline{f}_j-f_\rho)\|_K^2\right]
    \leq
    (4\tilde{C}_3+2\tilde{C}_4)|D|^{-\frac{2r-1}{2r+s}}.
$$
This completes the proof of Theorem \ref{Theorem:Optimal-Rate-adaptive}.
$\Box$
\end{proof}

\bibliographystyle{plain}
\bibliography{dis}
\end{document}